\documentclass{article}

\PassOptionsToPackage{round}{natbib}

\usepackage[preprint]{neurips_2026} 

\usepackage{mathtools} 
\usepackage{booktabs} 
\usepackage{tikz} 

\usepackage{algorithm}
\usepackage{algpseudocode}
\usepackage{graphicx}
\usepackage{subcaption}
\usepackage{hyperref}
\usepackage{wrapfig}
\usepackage{amsmath,amssymb,mathtools,amsthm,amsfonts}
\usepackage{nicefrac}       
\usepackage{microtype}      
\usepackage{xcolor}         
\usepackage{dsfont,bm}
\hypersetup{colorlinks=true, linkcolor = blue,
     anchorcolor = blue,
     citecolor = blue,
     filecolor = blue,
     urlcolor = blue}
\usepackage{soul}

\DeclareMathOperator*{\argmin}{arg\,min}
\DeclareMathOperator*{\argmax}{arg\,max}

\theoremstyle{plain}
\newtheorem{thmthm}{Theorem}

\newtheorem{thmprop}{Proposition}

\theoremstyle{definition}

\newtheorem*{thmrem*}{Remark}
\newtheorem*{thmprop*}{Proposition}

\newenvironment{proofsketch}{%
  \proof}{\endproof}

\theoremstyle{plain}
\makeatletter
\newtheorem*{rep@theorem}{\rep@title}
\newcommand{\newreptheorem}[2]{%
\newenvironment{rep#1}[1]{%
 \def\rep@title{#2 \ref{##1} (Restated)}%
 \begin{rep@theorem}}%
 {\end{rep@theorem}}}
\makeatother
\theoremstyle{plain}
\newreptheorem{theorem}{Theorem}
\newreptheorem{prop}{Proposition}
\newreptheorem{lemma}{Lemma}
\newreptheorem{cor}{Corollary}
\theoremstyle{definition}
\newreptheorem{def}{Definition}

\def\lobts{\textbf{lobTS}}

\def\lobucb{\textbf{lobUCB}}
\def\mts{\textbf{mTS}}
\def\mucb{\textbf{mUCB}}
\def\ts{\textbf{TS}}
\def\ucb{\textbf{UCB}}
\def\lob{\textbf{lob}}

\def\reg{\mathrm{Reg}}

\def\E{\mathbb{E}}

\def\V{\mathbb{V}}

\def\bmu{\boldsymbol{\mu}}
\def\blambda{\boldsymbol{\lambda}}
\def\hmu{\hat{\mu}}
\def\tmu{\tilde{\mu}}
\def\hbmu{\hat{\boldsymbol{\mu}}}
\def\tbmu{\tilde{\boldsymbol{\mu}}}

\def\bw{{\boldsymbol{w}}}

\def\bbP{\mathbb{P}}

\def\bbR{\mathbb{R}}

\def\cA{\mathcal{A}}

\def\cD{\mathcal{D}}

\def\cG{\mathcal{G}}
\def\cH{\mathcal{H}}
\def\cI{\mathcal{I}}

\def\cL{\mathcal{L}}
\def\cM{\mathcal{M}}
\def\cN{\mathcal{N}}
\def\cO{\mathcal{O}}
\def\cP{\mathcal{P}}

\def\cS{\mathcal{S}}

\newcommand{\KL}{\mathbb{KL}}
\newcommand{\lobTS}{\textbf{lobTS}}

\newtheorem{observation}{Observation}

\title{Latent Order Bandits}

\author{%
  Emil Carlsson \\
  Sleep Cycle AB\\
  \texttt{emil.carlsson@sleepcycle.com} \\
  \And
  Newton Mwai\\
  Chalmers University of Technology\\
  \& University of Gothenburg\\
  \texttt{mwai@chalmers.se} \\
  \And
  Fredrik D. Johansson \\
  Chalmers University of Technology\\
  \& University of Gothenburg\\
  \texttt{fredrik.johansson@chalmers.se} \\
}

\begin{document}

\maketitle

\begin{abstract}
    Bandit algorithms solve diverse sequential decision-making problems, but are often too sample-inefficient for from-scratch personalization. To substantially reduce exploration times, latent bandit algorithms exploit cross-instance structure implied by discrete latent states, provided that the posterior distribution of rewards and latent states is known and accurate. However, obtaining an accurate model of this structure is difficult, and a small number of latent states may be insufficient to characterize the reward distributions in all problem instances. We propose latent order bandits (LOB), relaxing the assumptions of latent bandits to require only prior knowledge of a \emph{partial order} of action preferences in each state. This allows instances of the same state to vary in reward distributions, as long as the partial order of actions is shared. For example, groups of users on a streaming service may agree on which movie genres are the best but rate experiences on different scales. We give an upper-confidence bound procedure for the LOB problem, applicable to both total and partial latent orders, and give an upper bound on its regret. To improve empirical performance, we propose a posterior-sampling algorithm and show, in a suite of experiments, that both are competitive with full-prior latent bandits when same-state instances share reward parameters, and preferable to them when reward scales differ between instances with the same latent state. 
\end{abstract}

%
%
\section{Introduction}
\label{sec:introduction}
Bandit algorithms offer a principled framework for decision-making under uncertainty, with strong theoretical guarantees~\citep{li2010contextual, chapelle2011empirical,Bouneffouf2020, yancey2020sleeping, o2022should}, but are often too inefficient for personalization~\citep{mwai2023fast}. For example, on a movie streaming platform, active users may be exposed to hundreds of recommendations during their subscription period, while others may interact with the system only a handful of times, insufficient to learn their preferences and improve their experience. Treating each user as its own bandit instance may result in poor recommendations and users leaving the service.




A common approach to using bandit algorithms for personalization is to leverage structural similarities between instances (users) to improve exploration times. Contextual bandits impose structure by modeling the expected rewards of actions as a shared function of observed context variables~\citep{lattimore2020bandit}, but these may be insufficient to fully characterize the reward distributions of instances with hidden preferences and require long exploration to learn. Clustering bandits~\citep{gentile2014online,li2019improved} address the latter by learning to group instances into clusters, assuming that the reward model of all instances in a given cluster share parameters. A related line of work, latent bandits~\citep{atan2018global,hong2020latent,maillard2014latent,pmlr-v206-pal23a, balciouglu2024identifiable}, exploit pre-learned reward models and use exploration to identify an instance's latent state, observed noisily through the context and the outcomes of actions. If the number of latent states is small, and the reward distribution can be learned effectively, this can substantially reduce exploration times~\citep{mwai2023fast}. 

A limitation of contextual, clustering and latent bandits, is that they assume that two instances would receive the same average reward for the same action, were they in the same context, cluster or state, respectively. When rewards represent subjective ratings, individuals (instances) tend to have different internal rating scales, unobserved by the learning algorithm. For contextual bandits, this leads to model misspecification, and for clustering/latent bandits, instance-specific scales demand a prohibitively large number of clusters, directly impacting the amount of exploration needed. Finally, most existing latent bandit algorithms require access to a full, accurate latent variable model, but give little or no guidance for how it can be learned~\citep{balciouglu2024identifiable}. Can we match the performance of latent bandits with fewer model constraints and less prior knowledge?


\paragraph{Contributions.} (i) We propose \emph{Latent Order Bandits} (LOB)---a variant of the latent bandit problem where each discrete latent state specifies a partial order of the mean rewards of actions, but \emph{not} a full distribution of rewards. Unlike existing latent bandits, this setting allows for structural similarity between instances (e.g., users) with different absolute reward scales but shared relative preferences. (ii) We apply an existing lower bound on regret to show that knowing the set of possible state orders can reduce the difficulty of the bandit problem (Section~\ref{sec:setup}).
(iii) We introduce a UCB algorithm (\lobucb{}) that exploits knowledge of the set of possible state orders, and give a bound on its regret. As a complement, we propose a heuristic Thompson sampling algorithm (\lobts{}) that exploits potential state orders to sample from an approximate posterior probability that each action is optimal (Section~\ref{sec:method}). (iv) We show empirically that \lobucb{}{} and \lobts{} can perform comparably to latent bandits with fully known reward distributions when both models are well-specified, and favorably to them when same-state instances have individual reward scales. Finally, in experiments on real-world rating data, we observe that both \lob{} algorithms degrade gracefully when misspecified (Section~\ref{sec:experiments}).


%
%
\section{Latent order bandits}
\label{sec:setup}
We study \emph{latent order bandits} (LOB) where an agent interacts with an environment over $T$ rounds. In each round $t \in [T]$, the agent selects an action (arm) $a_t \in \cA = \{1, ..., k\}$ and receives a stochastic reward $R_t$ from the environment. Our goal is to find a (possibly stochastic) policy $\pi$ for selecting actions $a_t$ based on the history of observations gathered by the agents so far, $\cD_t = ((a_1, r_1), ... (a_{t-1}, r_{t-1}))$, that results in the largest possible expected cumulative reward $\E[\sum_{t} R_t]$. 

Building on previous work on \emph{latent bandits} (LB)~\citep{maillard2014latent,hong2020latent}, we assume that each instance of the LOB problem instance is associated with a discrete latent state $s \in \{1, ..., m\}$, adding structure to the problem, but unknown to the agent. For a \emph{fixed} instance $i$ with density $p_i$, the reward at time $t$ depends only on the chosen action $a_t$, $\forall a, t'\neq t : p_i(R_t \mid A_t=a) = p_i(R_{t'} \mid A_{t'}=a)$. Thus, for any fixed instance $i$, we can index the rewards by actions, rather than rounds, $R_a \sim p_i(R_a) =  p_i(R_t \mid A_t = a)$. We leave out the subscript $i$ when clear from context.

The agent's goal is \emph{regret minimization} with respect to the unknown expectations of action rewards $\mu_{a} \coloneqq \E[R_a]$, the optimal action $a^* = \argmax_{a} \mu_a$ and the optimal reward $\mu^* = \mu_{a^*}$, with regret
\begin{equation}\label{eq:regret_min}
\reg(T) \coloneqq \sum_{t=1}^T\E_{A_t \sim \pi(\cD_t)}[\mu^* - R_t] = \sum_{t=1}^T\E_{\pi}[\mu^* - \mu_{A_t}]~.
\end{equation}
We denote the full vector of action reward mean parameters $\bmu = [\mu_1, ..., \mu_k]^\top \in \bbR^k$.

Existing \emph{latent bandits}~\citep{hong2020latent,maillard2014latent} exploit the assumption that the latent state $s_i$ of an instance $i$ is sufficient to determine the parameters $\bmu$ of its reward distribution, and hence, that any two instances $i,j$ that share state, $s_i = s_j=s$ also shared reward distribution, $p_i(R_a) = p_j(R_a) = p(R_a \mid  \bmu_s)$. In this case, if the conditional distribution of rewards $p(R_a \mid \bmu_s)$ and the parameters $\bmu_s$ of each state $s$ are known, $S$ can be inferred for a new instance and used to select the optimal action.\footnote{\citet{hong2020latent} also study context-dependent rewards. In this work, we focus on other aspects of the problem.}
Based on these assumptions, \citet{hong2020latent} proposed the latent bandit UCB algorithm \mucb{} and Thompson sampling~\citep{thompson1933likelihood,agrawal2012analysis} variant \mts{}, both proven to incur at most $O(\sqrt{mT\log T})$ cumulative regret in the worst case.

\subsection{Latent order bandits: Relaxing shared state parameters to partial order constraints }
The assumptions of latent bandits are hard to justify in practice. Assuming that all instances with the same state share the same reward distribution often leaves variation between instances unaccounted for, or forces the number of latent states to be very large. For example, if two users of a streaming service agreed on their preference \emph{order} of every single movie, they cannot have distinct \emph{rating scales} or must belong to different states. Yet, the relative preferences are sufficient to identify an optimal recommendation! Moreover, applying latent bandits requires that a full, accurate latent variable model is available but learning such a model is nontrivial~\citep{balciouglu2024identifiable}.

\begin{figure*}[t]
    \centering
    \begin{subfigure}{0.235\linewidth}
        \centering
        \includegraphics[width=1.\linewidth]{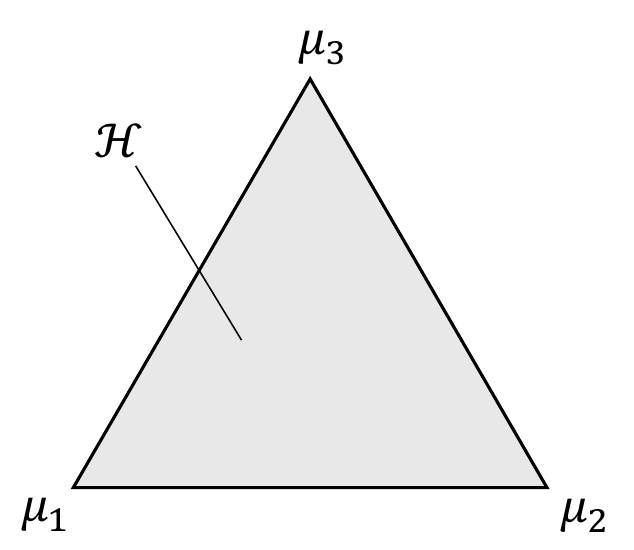}
        \caption{Multi-armed bandit (MAB). No constraints on reward means.}   
    \end{subfigure}
    \hfill
    \begin{subfigure}{0.235\linewidth}
        \centering
        \includegraphics[width=1.\linewidth]{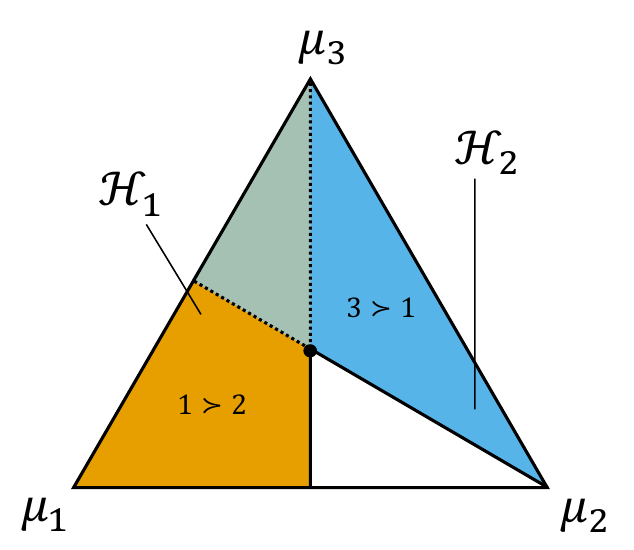}
        \caption{Latent order bandit (LOB) with \emph{partial} orders, unknown means.}    
    \end{subfigure}
    \hfill
    \begin{subfigure}{0.235\linewidth}
        \centering
        \includegraphics[width=1.\linewidth]{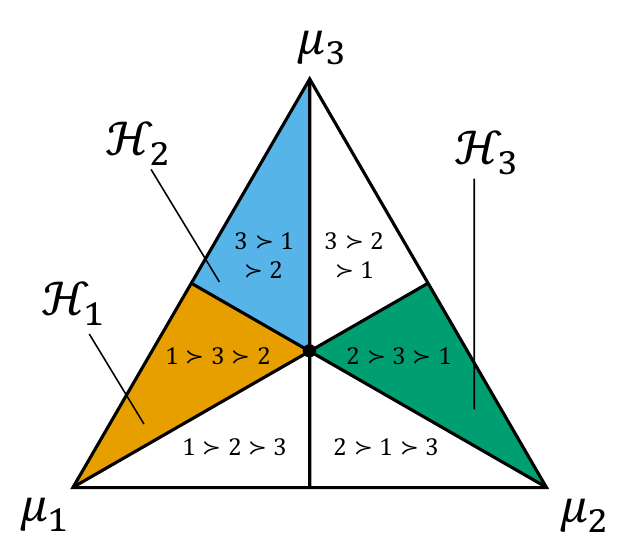}
        \caption{Latent order bandit (LOB) with \emph{total} orders, unknown means.}    
    \end{subfigure}    
    \hfill
    \begin{subfigure}{0.235\linewidth}
        \centering
        \includegraphics[width=1.\linewidth]{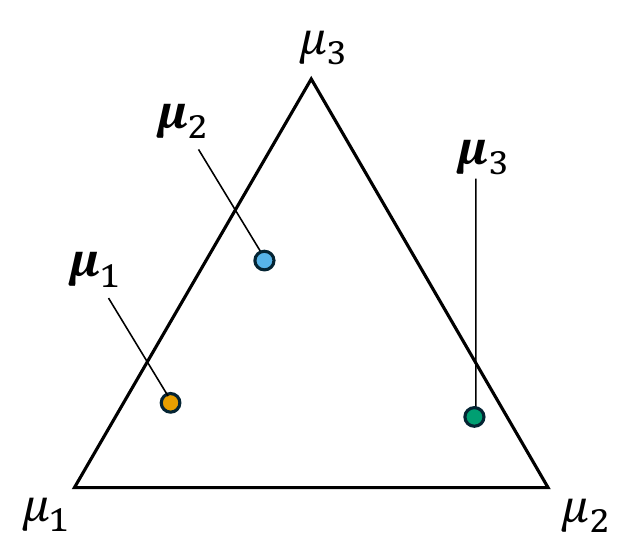}
        \caption{Latent bandit (LB): Fully known possible reward means.}    
    \end{subfigure}    
    \caption{The latent order bandit and related problems for reward means on the 2-simplex, $\bmu \in \Delta^{k-1}$. In the MAB problem, no structure is imposed. In LB, the full vector of reward means $\bmu_s$ is known a priori for each latent state $s$. In LOB, only the set of possible orders is known (colored segments).
    \label{fig:illustration}}
\end{figure*}

\paragraph{LOB instances.} We define an LOB problem instance by the unobserved latent state $s \in [m]$ and reward distributions $\cP = (P_1, ..., P_k)$ of the $k$ actions. For exposition, we focus on Gaussian rewards with equal variance $\sigma^2$ where $P_a = \cN(\mu_a, \sigma^2)$ and, for the remainder of the paper, represent an LOB instance as a tuple $(s, \bmu)$ of the latent state $s$ and reward means $\bmu$.\footnote{Our methods and analysis readily generalize to sub-Gaussian rewards. In Appendix~\ref{app:relative}, we briefly discuss problem variants with relative (dueling) feedback but do not consider that further here.}
Each state $s$ is associated with a known partial \emph{order} of actions, determined by constraints on their expected rewards $\bmu$. We define the order of state $s$ as a set $O_s = \{e_{s,j}\}_{j=1}^{n_s}$ of pairwise relations $e_{s,j} \in \cA \times \cA$ where $e_{s,j} = (a,a')$ indicates that $\mu_a > \mu_{a'}$ under $s$. $O_s$ is assumed reflexive, antisymmetric, and transitive for all $s$. We use $a \prec_s a'$ to denote that $\mu_a > \mu_{a'}$ can be inferred by the partial order of state $s$.

A problem instance $(s, \bmu)$ has arm parameters $\bmu \in \cH_s  \coloneqq \{\bmu\in \bbR^k : \forall (a,a') \in O_s, \mu_a >  \mu_{a'}\},
$
with $\cH_s$ the set of all parameters consistent with the partial order $O_s$. 
Two instances $(s, \bmu), (s', \bmu')$ with the same latent state, $s=s'$, are guaranteed to have the same partial order $O_s = O_{s'}$ but may have different reward parameters, $\bmu \neq \bmu'$. This allows for modeling individual rating scales, e.g., two users agree on their movie preferences, but not on what is the bar for ``five stars''. 
We assume that all latent states have mutually incompatible orders: $\forall s\neq s' : O_s \perp O_{s'}$, meaning that there exists at least one pair of arms $a, a'$ such that $a \prec_s a'$ and $a \succeq_{s'} a'$.
An order $O$ is \emph{total} if the relation between \emph{any} pair of arms $(a,a')$ can be inferred from $O$, i.e., the list of arms can be fully sorted according to their means. 
If fully determined by $O_s$, let $a^*_s$ denote the optimal arm for state $s$.



\label{sec:lowerbound}
At the start of learning, the agent does not know ($s, \bmu$) but is given access to the set of possible partial orders $\{O_s\}_{s \in \cS}$. This knowledge is critical to improve over unstructured bandits. We illustrate MAB, LB and LOB instances in Figure~\ref{fig:illustration} for the special case of reward means in the 2-dimensional simplex. \emph{Without} knowledge of $O_s$, the LOB problem would coincide with the classical MAB problem and is otherwise a strict generalization of the latent bandit problem. 
These differences directly influence an instance-specific lower bound on the asymptotic regret:
\begin{thmthm}[adapted from \citet{maillard2014latent,agrawal1989asymptotically}]\label{thm:lower_bound}
    Consider a problem instance $(s, \bmu)$ with sub-optimality gaps $\Delta_a=\mu^* - \mu_a$. For any uniformly good control scheme, i.e., that achieves $\reg(T) = o(T^b)$ for any $b>0$, \begin{equation}\label{eq:lower_bound}
    \liminf_{T\rightarrow \infty} \frac{\reg(T)}{\log T} \geq \hspace{-0.25em}\min_{\bw \in \cP(\cA)} \max_{\lambda \in \mathrm{Alt}(s,\bmu)} \hspace{-0.25em}\frac{ \sum_{a\in \cA}w_a \Delta_a}{\displaystyle \sum_{a\in \cA}\hspace{-0.5em} w_a \KL(\mu_a || \lambda_a)} 
    \end{equation}
    where $\mathrm{Alt}(s,\bmu)$ is the set of alternative problem instances, depending on the problem type (defined below). $w_a$ represent action proportions drawn from the probability simplex $\cP(\cA)$ on $\cA$. $\KL(\mu_a || \lambda_a)$ is the KL-divergence between unit-variance Gaussian distributions with  means $\mu_a, \lambda_a$.
\end{thmthm}

For standard MAB, the alternative set contains \emph{all} parameter vectors with a different optimal arm $\mathrm{Alt}_{\mathrm{MAB}}(\bmu) = \{\blambda : \argmax_a \lambda_a \neq \argmax_a \mu_a\}$, as instances are not associated with latent states. For latent bandits, $\mathrm{Alt}_{\mathrm{LB}}(s, \bmu) = \{\blambda^{s'}\}_{s' \neq s}$, where $\blambda^{s'} \in \cH_{s'}$ is a single \emph{fixed} parameter for each unrealized latent state $s'$. For LOB, the set of alternative instances contain only parameters $\mu'$ that have different optimal arms from $\bmu$ \emph{and} are consistent with the partial order of a state other than the true state's, $\mathrm{Alt}_{\mathrm{LOB}}(s, \bmu) = \{\blambda \in \cup_{s'}\cH_{s'}  : \argmax_a \lambda_a \neq \argmax_a \mu_a \}$. It follows that,
$$
\mathrm{Alt}_{\mathrm{LB}}(s, \bmu) \subseteq \mathrm{Alt}_{\mathrm{LOB}}(s, \bmu) \subseteq \mathrm{Alt}_{\mathrm{MAB}}(\bmu)~,
$$
\emph{and the corresponding lower bounds for the regret increase in the same order}, provided that the assumptions of each problem class hold. That is, the lower bound for standard latent bandits holds if instances of the same state have the \emph{same} reward distributions. If violated, even if the partial order constraint of LOB holds, the performance of LB algorithms can be arbitrarily bad, see Experiments.

When the union of state parameter sets imposes no constraints,  $\bigcup_{s'}\cH_{s'} = \bbR^k$, the LOB problem is analogous to an MAB problem without structure. This can happen if the partial order constraints are too few or the number of latent states is too large.
Conversely, if each $O_s$ prescribes a distinct total order, and $m \ll k!$, the closest confusing instance $\lambda^*$, maximizing \eqref{eq:lower_bound}, is likely to have an order with a large number of inversions to $\bmu$, the $\KL$ term will be large, and the bound small---the problem is easier to solve.  As a toy example, suppose that total orders $O_s$ are randomly selected without replacement from all possible permutations of $\cA$. Then, the probability that two such orders differ in only two positions is $\binom{k}{2}/(k! - 1)$ (see Appendix~\ref{app:similarity_orders}). For large $k$, this probability is vanishingly small. For example, with $k = 10$, this probability evaluates to $\frac{45}{3,628,799} \approx 1.24 \times 10^{-5}$. Given that $m$ is typically much smaller than $k!$, the ground-truth state will stand out even more as $k$ grows. 

%
%
\section{Algorithms}
\label{sec:method}
We design two algorithms for the LOB problem, one based on upper confidence bounds (\lobucb{}), for which we perform a regret analysis, and one based on posterior (Thompson) sampling (\lobts{}).

\begin{figure}[t]
    \centering
    \begin{minipage}[t]{0.48\textwidth}        
        \begin{algorithm}[H]
        \caption{Latent order UCB bandit (\lobucb{})}
        \label{alg:lobucb}
        \begin{algorithmic}[1]
        \State \textbf{Input:} Confidence parameter $c$ 
        \State $\forall a : N_a = 0$, $\hmu_a = 0$
        \State $\forall s : \mbox{Max}(s) = \{a : \not\exists a', a' \prec_s a\}$
        \vspace{0.27em}
        \For{$t=1, ..., T$}
           \For{$a \in \cA_t$}
                \State $B_a(t) = c\sqrt{4\log t/N_a(t)}$
                \State $L_a(t) \leftarrow \hbmu_a - B_a(t)$
                \State $U_a(t) \leftarrow \hbmu_a + B_a(t)$
            \EndFor
            \vspace{0.27em}
            \State $\cS_t = \{ s : \cH_s \cap \prod_a [L_a(t), U_a(t)] \neq \emptyset\}$  
            \State $\cA_t = \cup_{s \in \cS_t} \; \mbox{Max}(s)$
            \vspace{0.27em}
            \State Let $a_t \gets \argmax_{a\in \cA_t} U_a(t)$
            \State Play arm $a_t$, and observe $r_t$
            \State $\hmu_{a_t} \gets (N_{a_t}\hmu_{a_t}(t) + r_t)/(N_{a_t}+1)$
            \State $N_{a_t} \gets N_{a_t} + 1$
        \EndFor
        \end{algorithmic}
        \end{algorithm}
    \end{minipage}
    \hfill
    \begin{minipage}[t]{0.50\textwidth}
        \begin{algorithm}[H]
        \centering
        \caption{Latent order TS bandit  (\lobts{})}\label{alg:lobts}%
        \begin{algorithmic}[1]
            \State \textbf{Input:} Confidence parameter $c$ 
            \State $\cD_1 \gets \emptyset, \hbmu(1) \gets \boldsymbol{0}, \forall a: N_a \gets 0$
            \For{$a=1, ..., k$}
            \State $M_a \gets \{\bmu : a = \argmax_{a'} \mu_{a'}\}$ 
            \State $\cS_a = \{s : \not\exists a'\neq a : a' \prec_{s} a\}$
            \State $\Sigma_a \gets \cup_{s \in \cS_a}\cH_s, $ 
            \EndFor
            \For{$t=1, ..., T$}
            \For{$a=1, ..., k$}
                \State \quad $\tbmu^a \gets \displaystyle{\argmax_{\bmu \in M_a \cap \Sigma_a}} \cL(\cD_t  \mid \bmu, \bar{a}_t; c)$ 
            \State \quad Let $\hat{p}(A=a) \propto \cL(\cD_t  \mid \tbmu^a, \bar{a}_t; c)$
            \EndFor
            \State Sample action $a_t \sim \hat{p}(A)$
            \State Play arm $a_t$, and observe $r_t$
            \State {\small $N_{a_t} \gets N_{a_t} + 1,  \cD_{t+1} \gets \cD_{t} \cup \{(a_t, r_t)\}$}
            \EndFor
        \end{algorithmic}
    \end{algorithm}
    \end{minipage}
\end{figure}

%
%
\subsection{An upper-confidence bound latent order bandit algorithm}

In Algorithm~\ref{alg:lobucb}, we propose a UCB procedure for the LOB problem (\lobucb{}) that directly exploits knowledge of the latent state orders $O_s$ to improve exploration.\footnote{A previous version of this pre-print contained a preliminary version of lobUCB. This has since been removed.} The structure of \lobucb{} will be familiar from classical UCB algorithms~\citep{lattimore2014bounded}, but \emph{exploits the knowledge of feasible partial orders to dynamically restrict the active arm set} based on two observations.

Let $L(t), U(t) \in \bbR^{k}$ be vectors of lower and upper confidence bounds for the set of actions, given the current empirical means and play counts. With moving radius $B_a(t) = c\sqrt{4\log t / N_a(t)}$, define
$$
L_{a}(t) \coloneqq \hmu_a(t) - B_a(t)~,
\quad
U_{a}(t) \coloneqq \hmu_a(t) + B_a(t)~.
$$
First, any mean parameter $\mu$ that is both feasible, contained in $\cH \coloneqq \cup_s \cH_s$, and probable with high confidence falls in the intersection of $\cH$ and these confidence intervals. For example, a pair of arms $a, a'$ and a state $s$ such that $a \prec_s a'$ (a has higher mean under $s$) should satisfy $L_{a'}(t) \leq U_{a}(t)$, w.h.p. If there is a combination $(a, a',s)$ that violates this, $s$ is likely incorrect. We say that a state $s$ is \emph{feasible at $t$} if no such violation exists, and define the set of feasible states at $t$ as 
$$
\cS_t = \left\{ s : \cH_s \cap \prod_{a=1}^k [L_a(t), U_a(t)]\neq \emptyset \right\}~, 
$$
where $\prod$ indicates the cartesian product of the sets.
Second, any arm $a$ that cannot be maximal under any state $s$ need not ever be played. Define $\mbox{Max}(s) \coloneqq \{a : \not\exists a' , a' \prec_s a$\} as the set of potential maximal arms under $s$. For total orders, $\mbox{Max}(s)$ are singleton sets.

Combining the two observations above, we can dynamically restrict the arm set of a UCB algorithm to only consider arms at time $t$ that are potentially maximal under at least one feasible state at $t$, 
$$
\cA_t \coloneqq \{a : \exists s\in \cS_t, a\in \mbox{Max}(s)\}~,
$$
with $\cS_0 = [m]$ and $\cA_0$ the set of arms potentially maximal in some state based only on orders. In Algorithm~\ref{alg:lobucb}, we outline the \lobucb{} algorithm, performing UCB selection of arms from $\cA_t$. 

The regret of \lobucb{} varies with the dynamic number of potentially maximal arms $|\cA_t|$, bounded from above by $q \coloneqq |\cA_0|$. If all state orders $O_s$ are total, $q = \min(k, m)$ since every state has exactly one maximal arm and there are only $k$ arms to choose from. In all cases, $q \leq \min(k,m)$. Below, we bounded the regret of \lobucb{} based on this quantity. 
\begin{thmthm}\label{thm:ucbregret}
    Consider an LOB instance $(s^*, \bmu^*)$ with $k$ arms and $m$ states. Suppose that each possible state $s$ has a partial order $O_s$ and unique optimal arm $a^*_s$. Let $q=|\cA_0| \leq \min(k,m)$. Assume rewards are sub-Gaussian with variance bounded by $\sigma^2$ and define $\Delta_a = \mu^*_{a^*} - \mu^*_a$ with $\Delta_a > 0$ for $a \neq a^*_{s^*}$. Then, Algorithm~\ref{alg:lobucb} using $c=\sigma$ has expected instance-dependent regret
    $$
    R_T \leq \cO(\sum_{a\in \cA_0} \frac{\sigma^2 \log T}{\Delta_a}+q\Delta_{\mathrm{max}}
    )
    $$  
    and worst-case instance-independent regret,
    $$
    R_T \leq \cO(\sigma \sqrt{\min(k,m)T\log T}).
    $$      
\end{thmthm}
\begin{proofsketch}
    Our proof involves reasoning about two events, either of which occurs at every round $t$: i) The good event $\cG$, on which the true state $s^*$ is feasible, and ii) the bad event $\cG^c$, where $s^*$ may be discarded. On the good event, \lobucb{} performs standard UCB selection over $|\cA_t| \leq q$ arms, and benefits from the same analysis. On the bad event, in the worst case, the arm with gap $\Delta_{\mathrm{max}} = \max_{a\neq a^*} \Delta_a$ is chosen. The bad event contributes only constant regret. 
\end{proofsketch}

A full proof is given in Appendix~\ref{app:proof_lob_ucb}. Theorem~\ref{thm:ucbregret} shows that the regret of \lobucb{} is sublinear in $T$, scaling as $O(\sqrt{\min(k,m) T\log T})$ in the (instance-independent) worst case. Disregarding log-factors, this matches both the best achievable regret of the unstructured MAB problem, $\Theta(\sqrt{kT})$, reached by, e.g., upper-confidence bound (UCB) maximization and posterior (Thompson) sampling methods~\citep{lattimore2020bandit}, and results for the standard latent bandit problem. For the latter, \citet{hong2020latent} showed that the worst-case asymptotic regret of the  \mts{} and \mucb{} algorithms is $O(\sqrt{mT\log T})$. However, it should be noted, that when all states have total orders and distinct optimal arms (consistent with LB assumptions), simply restricting the arm set to the $\min(k,m)$ arms that are optimal in one of the states and running a standard MAB algorithm achieves $\Theta(\sqrt{\min(m,k)T})$ worst-case regret, (see Appendix~\ref{app:topm}). Algorithm~\ref{alg:lobucb} is conservative in that it does not exploit the joint likelihood of the observed order of multiple arms. 

Empirically, even in problems with $m \gg k$, knowledge of possible partial orders $\{O_s\}_{s\in [m]}$ (as in \lobucb{}) or the full posterior distribution $p(S \mid \cD_t)$ (as in \mts{}) can massively improve empirical performance over unstructured MAB algorithms, as we see in Section~\ref{sec:experiments}. For LB algorithms, the strong prior knowledge means that no parameter estimation, only inference, is performed for a new instance. For LOB algorithms, parameter estimation is still necessary since $\bmu$ is determined only up to a partial order given $s$, but its order constraints restrict the search space, as discussed in Section~\ref{sec:lowerbound}. 
 
%
%
\subsection{A posterior sampling LOB algorithm: \lobts{}}
\label{sec:thompson}

Although sublinear in worst-case regret, \lobucb{} can be improved upon further empirically by making better use of the likelihood of the observed rewards under each state. To this end, we propose a Thompson sampling heuristic (\lobts{}, Algorithm~\ref{alg:lobts}) that samples arms based on optimistic estimates of the posterior probability of the mean parameters under which each arm $a$ is optimal.

Recall that $\cD_T = ((a_1, r_1), ..., (a_T, r_T))$ denotes the data set of the first $T$ observations collected in a problem instance $(s, \bmu)$. The likelihood of the rewards in $\cD_T$ under $s$ with latent order $O_s$ is then
\begin{equation}\label{eq:likelihood}
\cL(\cD_T \mid S=s, \bar{A}_T = \bar{a}_T) = \prod_{t=1}^T p(r_t \mid a_t, s) = \int_{\bmu \in \cH_s} p(\bmu \mid s) \prod_{t=1}^T  p(r_t \mid a_t, \bmu, s) d\bmu~,
\end{equation}
where $\bar{A}_t$ and $\bar{a}_t$ are the sequences of random variables and observations of actions, respectively. In \mts{}, knowledge of each state's mean parameters allows for computing the posterior probability of latent states $p(S=s \mid \cD_t, \bar{a}_t)$ using \eqref{eq:likelihood}. A state $s_t$ is sampled, and the optimal arm $a_{s_t}^*$ under that state is played. In LOB, since $s$ only prescribes order constraints on $\bmu$, we must marginalize over the set of mean parameters that comply with $\cO_s$, but this means that computing \eqref{eq:likelihood} is intractable. We could appeal to Monte-Carlo sampling or variational inference~\citep{jordan1999introduction} for an approximation, provided that a parameter prior $p(\bmu \mid s)$ is known for each latent state $s$. However, in general, (i) no state parameter prior $p(\bmu \mid s)$ is available, and (ii), the optimal arm $a^*$ may not be known \emph{even if the true state is known}, since $s$ may only prescribe a \emph{partial} order to the arms.

Our strategy in \lobts{} is to compute an approximate posterior probability that a given arm $a$ is optimal, and sample arms to play according to that, based on optimistic estimates of the likelihood given the optimality of the arm. That $a$ is optimal, $a^* = a$, implies that $\bmu$ lies in the intersection of $M_a \coloneqq \{\bmu : a = \argmax_{a'} \mu_{a'}\}$, and the union of parameter sets for states $s$ under which $a$ \emph{could be} optimal, $\Sigma_a \coloneqq \cup_{s \in \cS_a}\cH_s$ where $\cS_a = \{s : \not\exists a'\neq a : a' \prec_{s} a\}$. Thus, 
\begin{align}\label{eq:likelihood_ub}
\cL(\cD_T \mid a^*=a, \bar{a}_T) \leq \sup_{\bmu \in M_a \cap \Sigma_a} \prod_{t=1}^T  p(r_t \mid a_t, \mu_{a_t}) = \prod_{t=1}^T  p(r_t \mid a_t, \tmu^a_{a_t}) = \cL(\cD_T \mid \tbmu^a, \bar{a}_T) ~.
\end{align}
where $\tbmu^a = \argmax_{\bmu \in M_a \cap \Sigma_a} \cL(\cD_t  \mid \bmu, \bar{a}_t)$. For Gaussian likelihoods with uniform variance, the negative log-likelihood 
$-\log \cL(\cD_t  \mid \bmu, \bar{a}_t) \propto \sum_{a}N_a(t) (\mu_a - \hmu_a(t))^2 \eqqcolon \|\bmu - \hbmu(t)\|^2_{N(t)}$, where $\hbmu(t)$ are the empirical reward means of each arm observed in $\cD_t$. Since $M_a \cap \Sigma_a = \cup_{s\in \cS_a}(M_a \cap \cH_s)$, we can compute $\tbmu$ by projecting $\hbmu$ onto $M_a \cap \cH_s$ for all $s$ under $\|\cdot\|^2_{N(t)}$, and letting
$$
\cdot, \tbmu^a  = \argmin_{s, \bmu \in M_a \cap \cH_s} \|\bmu - \hbmu(t)\|^2_{N(t)}~\mbox{for all}~a.
$$
Thus, computing $\tbmu^a$ can be done by solving multiple projection problems, just like for $\lobts{}$.
Based on the \emph{optimistic} bound on the likelihood in \eqref{eq:likelihood_ub}, we construct an approximate posterior probability for each arm being optimal, 
$\tilde{p}(a^*= a\mid \cD_t, \bar{a}_t) \propto \cL(\cD_t \mid \tbmu^a, \bar{a}_t)$. At each round of \lobts{}, we sample the arm $a_t$ to play from this distribution, completing Algorithm~\ref{alg:lobts}.
Finally, we note that this approximate posterior probability is not unbiased in general, since the likelihood bound $p(\cD_t \mid \tbmu_s) \geq p(\cD_t \mid s)$ affects also the normalization constant of the posteriors: The looser the bound for one state, the lower the estimated posterior probability of other states.

\paragraph{On the regret of \lobts{}.}
\label{sec:scaling}
We do not yet have a formal analysis of the regret of \lobts{}. Like for \lobucb{}, we expect the performance of \lobts{} to be affected by the nature and number of active order constraints resulting from the projections to each parameter set (ln.8 of Algorithm~\ref{alg:lobts}). As data accumulates in $\cD_t$, the projection distance to parameters complying with the true state $s^*$ will approach 0, while the distances to other states increase. Unlike \lobucb{}, \lobts{} exploits that the more order constraints differ between states, the more the playing of \emph{any} arm will distinguish the true latent state in terms of its likelihood of observed data. Indeed, we observe the effects of this empirically in Figure~\ref{fig:q_sweep}, where we successively increase the number of partial order constraints, and note a sharp decrease in the cumulative regret of \lobts{}, faster than for \lobucb{}.

\subsection{Time complexity of \lobucb{} and \lobts{}}
\label{sec:complexity}
The arm indices of \lobucb{} can trivially be computed in each round in $O(km)$ time. For \lobts{}, in each round, the mean projections (ln.8) can be performed by solving multiple partial isotonic regression problems ~\citep{barlow1972isotonic}, see Appendix~\ref{app:isotonic}. This is done once per action $a$ and state in which action $a$ could be optimal. Solving these problems dominate the time complexity of \lobts{}. For total orders of $k$ items, isotonic regression can be solved in $O(k)$ time, resulting in an (at worst) $O(mk^2)$ per-round complexity for \lobts{}. For partial orders, isotonic regression can be reduced to a maximum flow problem on a graph with $\kappa_s$ edges, where $\kappa_s = |\cO_s|$ is the number to the number of partial order constraints. This problem can, ignoring log factors, be solved in $\tilde{O}(\kappa^{3/2})$ time. In practice, we find that only performing the projection every $T_{proj}=10$ rounds,  sampling arms and updating parameters as normal, does not hurt the performance of either algorithm (see Appendix Figure~\ref{fig:pevery}). This heuristic speeds up the overall training by a factor $T_{proj}$.

%
%
\section{Experiments}
\label{sec:experiments}
We evaluate our algorithms \lobucb{} and \lobTS{} by comparing their performance to unstructured MAB algorithms \ucb{} and Thompson sampling with a Gaussian prior (\ts{})~\citep{lattimore2020bandit}, as well as latent bandit algorithms \mucb{} and \mts{} from \citet{hong2020latent}, which require access to the set of full parameter vectors $\{\bmu_s\}_{s=1}^m$ and the complete reward likelihood under each state as prior knowledge. We focus on estimating the value of using order constraints compared to no prior structure (\ucb{}, \ts{}) and to full latent variable models (\mucb{}, \mts{}). Our evaluation environments include diverse synthetic tasks and a semi-synthetic task compiled from the MovieLens data set~\citep{harper2015movielens}. All methods use the true reward variance as their uncertainty parameter. For a closer description of parameters and baseline methods, see Appendix~\ref{app:experiments}.

Each experiment is repeated $N$ times with independently sampled orders and reward parameters over a horizon of $T$ rounds, chosen based on the environment. We use instantaneous and cumulative regret sequences, and average regret at a fixed horizon $T$ as error metrics, reporting averages over $N$ independent instances with standard errors $\sigma/\sqrt{N}$, where $\sigma$ is the standard deviation, as error bars. 

\subsection{Synthetic environments}
First, we design a synthetic environment that is well-specified for both LB and LOB, with varying numbers of arms, $k$, and latent states, $m$. For each state, $s \in [m]$, a random \emph{total} order (permutation) $\bbP_s = (p_{s,1}, ..., p_{s,k})$ of the arm set $\cA = [k]$ is sampled without replacement, to impose constraints $\mu_{p_{s_1}} \geq \mu_{p_{s_2}} \geq ... \geq \mu_{p_{s_k}}$. Thus, $O_s$ contains the $k-1$ pairwise relations $O_s = \{(p_{s,a}, p_{s,a+1})\}_{a=1}^{k-1}$. The reward for each arm $a$ in $s$ is Gaussian $R_{s,a} \sim \cN(\mu_a, \sigma^2)$ with mean decreasing according to the permutation, $\mu_{p_{s,j}} = k\gamma - j\gamma$, where $\sigma, \gamma=1$ are constants. To simulate knowledge of partial orders, a random subset of $\lceil q(k-1) \rceil$ pairwise relations in $O_s$ are sampled and revealed to the \lob{} algorithms for an experiment parameter $q \in [0, 1]$.
Highlighted results are presented in Figure~\ref{fig:exp_synth_overview}.

\begin{figure*}
    \centering
    \begin{subfigure}{0.32\textwidth}
        \includegraphics[width=\textwidth]{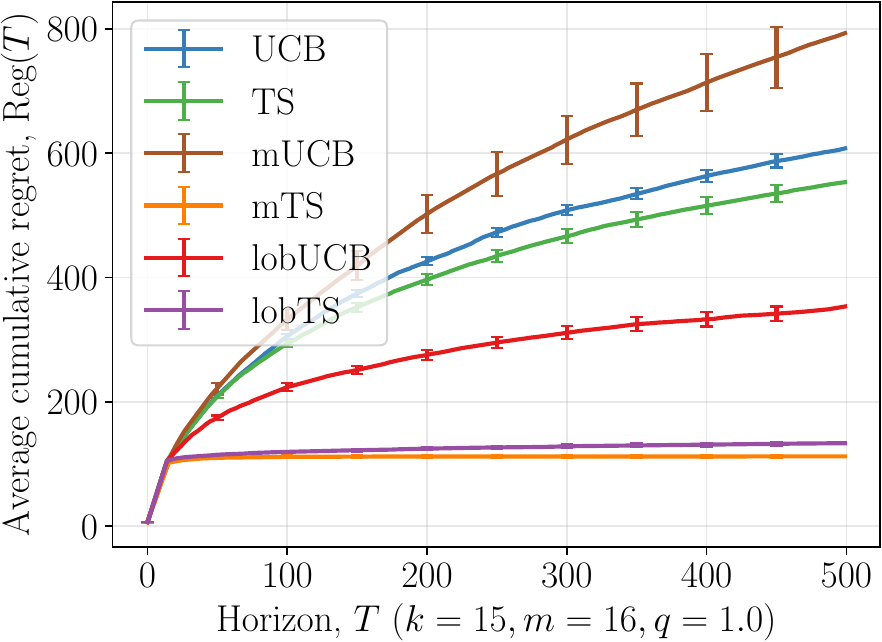}
        \caption{Regret vs horizon\label{fig:cum_regret_example}}%
    \end{subfigure}%
    \;
    \begin{subfigure}{0.32\textwidth}
        \includegraphics[width=\textwidth]{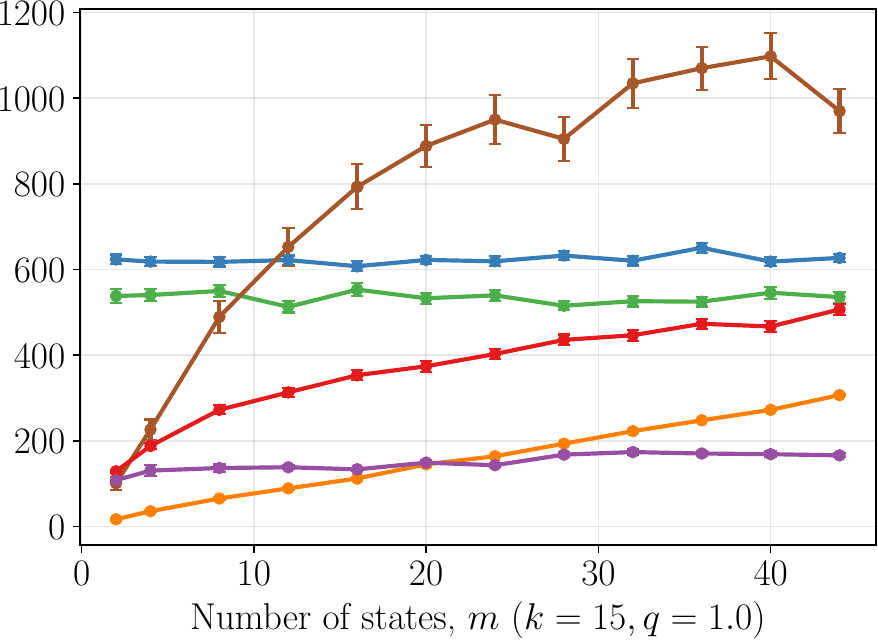}
        \caption{Regret vs \#states at $T=500$ \label{fig:m_sweep}}%
    \end{subfigure}%
    \;
    \begin{subfigure}{0.32\textwidth}
        \includegraphics[width=\textwidth]{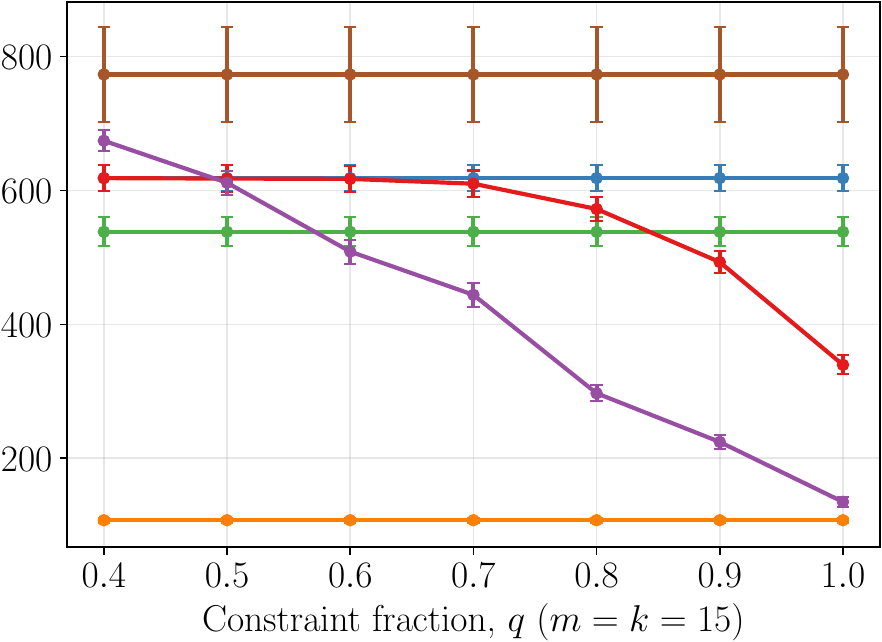}
        \caption{Regret vs constraints\label{fig:q_sweep} at $T=500$}%
    \end{subfigure}%
     \caption{Results on \textbf{Synthetic environments}. Average cumulative regret $\reg(T)$ compared to baselines as a function of $T$ (a) or at fixed $T=500$ (b,c) for different arm and state configurations. In (c), \lob{} algorithms\label{fig:cumulative_reward} receive partial order information containing a fraction $q$ of the total order constraints. Despite using less prior knowledge, and all methods being well-specified, \lobucb{} and \lobts{} consistently outperform \mucb{} and classical methods \ucb{} and \ts{}. 
     \label{fig:exp_synth_overview}}
\end{figure*}

\paragraph{Latent order bandits benefit from structure, improve with more constraints.}
In Figures~\ref{fig:m_sweep}--\ref{fig:q_sweep}, we see that \lobucb{} and \lobts{} consistently outperform unstructured baselines \ucb{} and \ts{} when sufficient structure is available (constraint fraction $q>0.6$ in Figure~\ref{fig:q_sweep}). Except for very small numbers of states ($m=2$), both \lob{} algorithms achieve lower average regret than \mucb{}, despite the latter requiring access to the full reward parameter for each latent state, not just a partial order (see Figure~\ref{fig:m_sweep}). In Figures~\ref{fig:m_sweep}--\ref{fig:q_sweep}, we see that for $m\geq k=15$, when \lobts{} is given access to a total order ($q=1$), the performance of \lobts{} can approach, and even overtake that of \mts{}. However, this comes down primarily to \mts{}'s longer initial exploration phase as $m$ grows (see Appendix Figure~\ref{fig:cumulative_regret_synthetic}).

\subsection{MovieLens experiments}
We construct a realistic task based on the MovieLens data sets~\citep{harper2015movielens} where instances represent users, actions $a_t$ represent choices of movie \emph{genre} ($k=19$), rewards $r_t$ represent ratings of a movie within the selected genre, and latent states $s$ represent groups of users with similar movie preferences. 
For each user $i$ in the data set, we compute their individual average rating $\mu^{0,i}_a$ of movies within each genre $a$. These user-specific genre means are clustered into latent states using k-means clustering with Kendall's Tau~\citep{kendall1938new} as distance and $m$ centroids, varied across experiments. Each resulting state $s$ has a total order $O_s$, corresponding to the centroid's mean parameters. To create a well-specified setting (A) for LOB, each user's means $\bmu^{0,i}$ are projected to comply with the order of their assigned cluster $s_i$, based on Euclidean distance to $\cH_{s_i}$, forming parameters $\bmu^i$. We also add a minimum gap $\gamma=0.01$ between the means of any two arms (genres). Rewards for an instance $i$ are generated as $r^i_t \sim \cN(\mu_{a_t}^i, \sigma^2)$ with $\sigma=2$. As a  misspecified setting (B), we leave the instance reward means as the original rating averages, $\bmu^i = \bmu^{0,i}$, which may not agree with any of the orders $\{O_s\}$, for any value of $m$. In both settings, for the UCB indices and posterior state probabilities of \mucb{} and \mts{}, respectively, we use the state centroids' mean rewards, $\bar{\bmu}_s = \sum_{i : s_i = s} \bmu_i$. Note that \mucb{} and \mts{} are \emph{misspecified} in both Settings A and B as \emph{individual rating scales differ within states} even if they \emph{agree on the order of arms} (movie genres). 

For the bandit evaluation, we use the first $N=100$ users from the data set and perform $T=4000$ rounds of choosing a movie genre with a random reward (rating) per user. The number $T$ is chosen intentionally \emph{very} large compared to what we may find in a real-world setting to allow unstructured bandits a chance to converge. See Appendix~\ref{app:movielens} for a full description of the MovieLens environment.

\begin{figure*}
    \centering
    \begin{subfigure}{0.48\textwidth}
        \includegraphics[width=\textwidth]{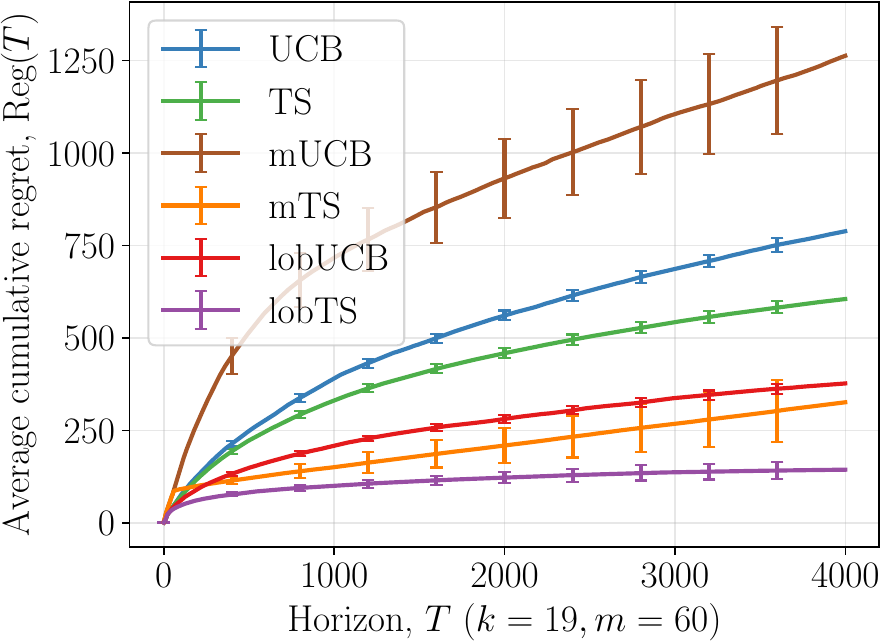}
        \caption{Regret vs horizon\label{fig:movie_regret} (well-specified)}%
    \end{subfigure}%
    \;
    \begin{subfigure}{0.48\textwidth}
        \includegraphics[width=\textwidth]{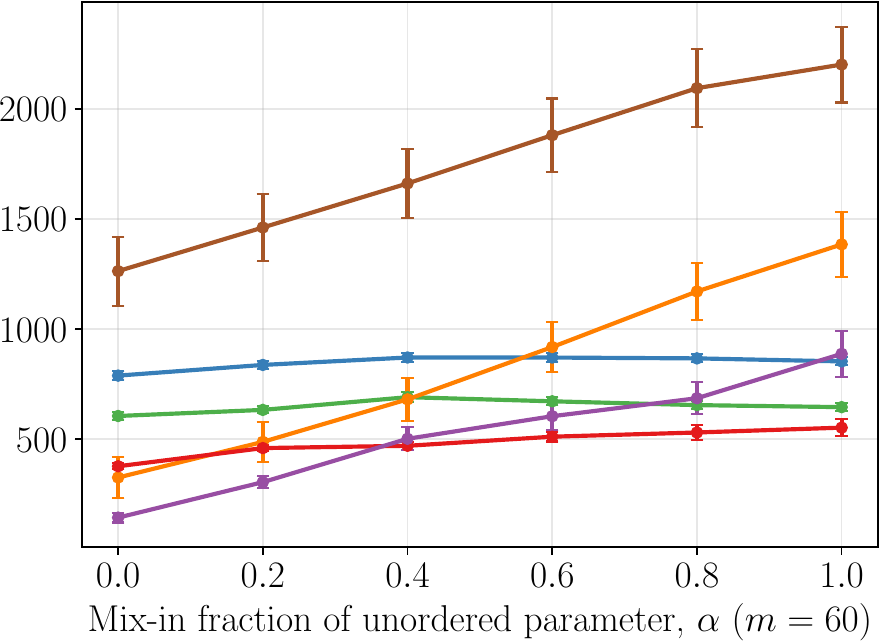}
        \caption{Regret at $T=4000$ vs misspecification factor $\alpha$ \label{fig:movie_mixin}}%
    \end{subfigure}%
    \caption{Results on \textbf{MovieLens} (A). In (a), instances are well-specified, with parameters ordered according to one of the $m=60$ states constructed by clustering the MovieLens rating data. In (b), for each instance (user) the ground-truth genre rating means $\bmu^0$ are mixed with the well-specified means $\bmu$ to form $\bmu^{\mbox{mix}} = \alpha \bmu^0 + (1-\alpha)\bmu$, to test robustness to misspecification. \label{fig:movielens}}
\end{figure*}

\begin{figure}
    \centering
    \begin{subfigure}{0.48\textwidth}
        \includegraphics[width=\textwidth]{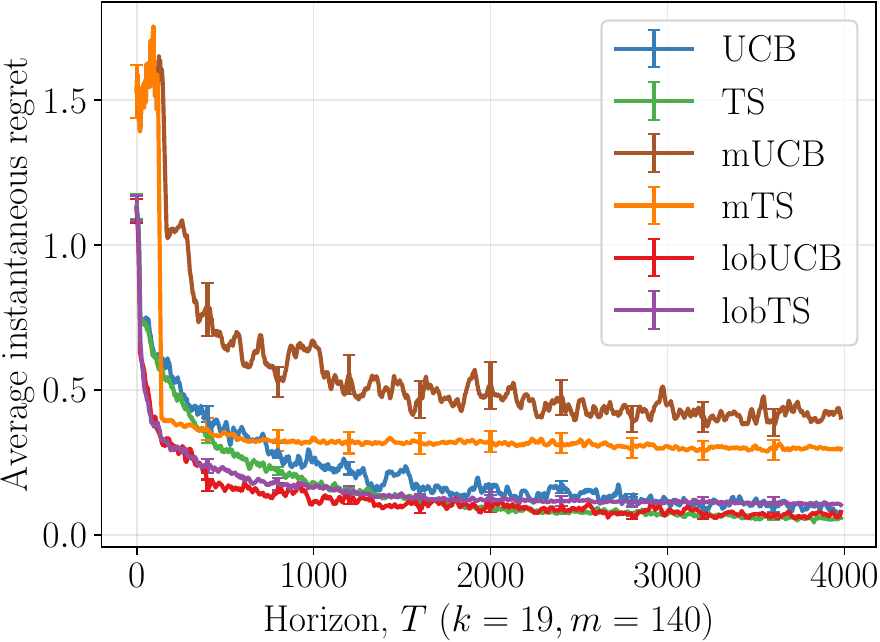}
        \caption{Instantaneous regret on \textbf{MovieLens} B (misspecified), using real-world user means $\bmu^{0}$ for rewards with $m=160$. Moving average with window size 20 for readability. \label{fig:movie_m160_sweep_ireg}}
    \end{subfigure}
    \quad
    \begin{subfigure}{0.48\textwidth}
        \includegraphics[width=\textwidth]{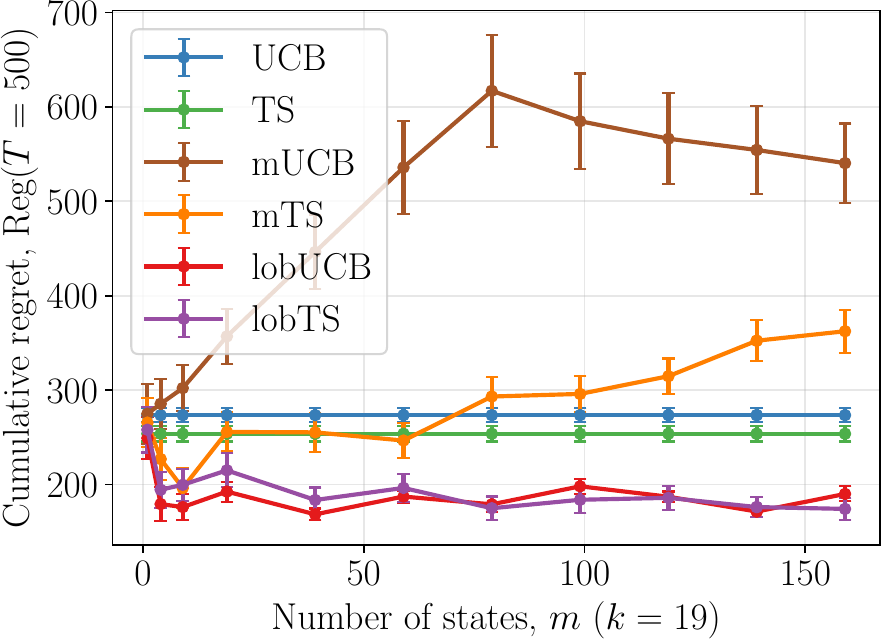}
        \caption{Cumulative regret on \textbf{MovieLens} B (misspecified) at $T=500$, using real-world user means $\bmu^{0}$ for varying $m$.\label{fig:movie_B_m_sweep} \vspace{1em}}%
    \end{subfigure}
\end{figure}

\paragraph{LOB exploits structure while same-state rewards differ.}
In Figure~\ref{fig:movie_regret}, for the well-specified Setting A with $m=60$ states, we see that latent order bandits \lobucb{} and \lobts{} both quickly converge to low instantaneous regret (small slope, see also Appendix Figure~\ref{fig:instantaneous_regret_movies_m60}), with \lobts{} comfortably outperforming all other methods. \mts{} quickly achieves low instantaneous regret but is misspecified due to individual reward means differing within the same state, and its cumulative regret grows linearly with $T$. Similar behavior is observed for $m \in \{10, 20, 80\}$ (see Appendix Figure~\ref{fig:cumulative_regret_movies}).

\paragraph{LOB algorithms are robust to misspecification.} We test the robustness of \lob{} algorithms to misspecification of the partial latent orders $O_s$ by gradually \emph{mixing} the well-specified means $\bmu^i$ of Setting A with the misspecified means $\bmu^{0,i}$ of Setting B, forming $\bmu^{\mbox{mix}, i} = \alpha\bmu^{0, i} + (1-\alpha)\bmu^{i}$ for $\alpha \in [0,1]$. The larger the $\alpha$, the more misspecification. Results are shown in Figure~\ref{fig:movie_mixin}. 
For the sequence of $\alpha \in (0, 0.2, 0.4, 0.6, 0.8, 1.0)$, the average number of violated order pairwise constraints in $\bmu^{\mbox{mix},i}$ is $(0.0, 6.8, 7.17, 7.35, 7.39, 7.39)$. Despite the rapid increase in misspecified constraints, the cumulative regret increases smoothly, with both \lobucb{} and \lobts{} improving over unstructured bandits for low-to-moderate misspecification ($\alpha \leq 0.6$), see Figure~\ref{fig:movie_mixin}. \mts{} is competitive at $\alpha=0$, despite its violated assumptions, but otherwise incurs higher regret than LOB algorithms.

In Figure~\ref{fig:movie_m160_sweep_ireg}, we see that with the misspecified means of Setting B, a small bias remains for LOB algorithms (see $T>3000$), but substantially smaller than for \mucb{} and \mts{}. However, in small-horizon settings ($T=500$), clustering users by a small number of states $m$ is sufficient for latent order bandits to outperform unstructured bandits on the real-world rating means, and remain preferable to all baselines as $m$ grows (Figure~\ref{fig:movie_B_m_sweep}). This clearly demonstrates the promise of our approach.

\paragraph{Computation time.} The biggest drawback of \lobts{} is computation time. On  \textbf{MovieLens} with $m=140$, completing $T=4000$ rounds for $N=100$ instances took 30min for \lobucb{} and 94min for \lobts{}, compared to 3min for \mts{} and 1min for \mucb{}. However, we stress that in real-world personalization tasks, like movie recommendation, the bottleneck is not the model update time, (here, $\leq$2s per round), but the time to obtain each reward (e.g., a user watching a movie).

%
%
\section{Discussion}
\label{sec:discussion}
We have proposed \emph{latent order bandits} (LOB), adding structure to multi-armed bandits through latent states, defined by partial order constraints on reward parameters. This relaxation of latent bandits allows instances of the same state to have individual reward distributions. We designed regret minimization algorithms \lobucb{} and \lobts{} and showed that the former achieves regret sublinear in the horizon, $T$, scaling in the worst case as $\cO(\sqrt{\min(k,m)T\log T})$. In experiments, both algorithms improve over unstructured bandits and, despite using less prior knowledge, are competitive with latent bandits that use a complete latent-variable model~\citep{hong2020latent}, when these are well specified. When same-state instances differ in  parameters, LOB methods perform the best.

Our setting is related to structured bandits~\citep{lattimore2014bounded,combes2017minimal} and bandits with reward constraints~\citep{pacchiano2021stochastic, Camilleri22,carlsson2024pure} or clustered arms~\citep{Pandey2007,Zhao19, carlsson2021thompson}, but also more broadly to contextual bandits~\citep{chu2011contextual, agrawal2013thompson, zhou2015survey, lattimore2020bandit}, as these also exploit structure between context, actions, and rewards, promising high personalisation. Our setting is also conceptually related to bandits with preference feedback and dueling bandits~\citep{yue2009interactively, sui2018advancements, bengs2021preference, herman2024}, a bandit class where action pairs are selected at each round, and the rewards are independent stochastic preference feedback about arm preference~\citep{sui2018advancements,ailon2014reducing,bengs2021preference}. Our setting differs in that actions are single arms (not pairs) and rewards are absolute.

\paragraph{Limitations.} Our current theoretical understanding of latent order bandits is limited to the general lower bound on the regret in \eqref{eq:lower_bound}, which decreases as order constraints on the mean parameters are added, and the upper bound for the regret of \lobucb{} in Theorem~\ref{thm:ucbregret}. This upper bound matches both classical MAB and latent bandits. It does not, however, reveal much about the structure of order constraints affect the performance of the algorithm. Both of our algorithms, \lobucb{} and \lobts{} are limited to non-contextual bandits, but if the state preference orders for each context are known, the extension is immediate. Recent works explore non-stationary latent bandits where the latent variable evolves over time~\citep{hong2020non, nelson2022linearizing, russoswitching}, which is not considered here. However, a common theme in these works is that the reward distributions are determined completely by the latent state. This differs from our setting in the sense that we assume the latent state only defines an \emph{order} of the arms.


\begin{ack}
We thank Ahmet Zahid Balc{\i}o\u{g}lu for his valuable input on the manuscript. The work was partially supported by the Wallenberg AI, Autonomous Systems and Software Program (WASP) funded by the Knut and Alice Wallenberg Foundation.

The computations and data handling were enabled by resources provided by the National Academic Infrastructure for Supercomputing in Sweden (NAISS), partially funded by the Swedish Research Council through grant agreement no. 2022-06725.
\end{ack}

{\small
\bibliographystyle{plainnat}
\bibliography{main}
}

%
%
\clearpage
\appendix
\onecolumn

\section*{Notation}
A list of common notations is given in Table~\ref{tab:notation}. Generally, capital Roman letters denote random variables and lower-case Roman letters denote constants or observations or random variables. The sequence $O_s$ is an exception since it is not random. Caligraphic Roman letters denote sets. 
\begin{table}[H]
    \centering
    \caption{Common notation}
    \label{tab:notation}   
    \vspace{0.5em}
    \begin{tabular}{ll}
    \toprule
        $k$ & Number of arms \\
        $m$ & Number of latent states \\
        $t, t'$ & Time indices \\
        $T$ & Time horizon \\
        $\cA$ & Set of available actions \\
        $a$ & A single action in $\cA$. $a_t$ is the action at time $t$\\
        $N_a(t)$ & Number of plays of arm $a$ at time $t$ \\
        $\cS$ & Set of latent states \\
        $S$ & Latent variable on $\cS$ \\
        $s$ & A single latent state in $\cS$ \\
        $R_a$ & Stochastic reward for action $a$ \\
        $r_t$ & Observation of reward at time $t$ following action $a_t$ \\
        $\mu_a$ & Expected reward under action $a$, $\mu_a = \E[R_a]$ \\
        $\sigma_a$ & Standard deviation of reward under action $a$, $\sigma^2_a = \V[R_a]$ \\
        $\bmu$ & Vector of reward means for all actions \\
        $a^*$ & Action with the highest expected reward  \\
        $\mu^*$ & Optimal reward, reward of optimal action \\
        $\mu_{s,a}$ & Expected reward under action $a$ in latent state $s$ \\
        $\sigma_a$ & Standard deviation of reward under action $a$ in latent state $s$\\
        $\bmu_s$ & Vector of reward means in latent state $s$ for all actions \\
        $\cH_s$ & Set of valid reward means in latent state $s$, $\bmu_s \in \cH_s$ \\
        $\cH$ & Set of globally valid reward means, $\bmu \in \cH = \cup_{s \in \cS} \cH_s$\\
        $a_s^*$ & Action with the highest expected reward in latent state $s$ \\
        $\mu_s^*$ & Optimal reward for any action in latent state $s$ \\
        $\hmu_a, \hbmu, \hmu_{s,a}, \hbmu_s$ & Estimates of reward means corresponding to the above \\
        $\tmu_a, \tbmu, \tmu_{s,a}, \tbmu_s$ & Projected empirical means corresponding to the above \\
        $\pi$ & Decision making policy to select $a_t$ \\
        $\reg(T)$ & Cumulative regret until horizon $T$ \\
        $\cM$ & Latent-variable model \\
        $O_s$ & Partial order of actions for latent state $s$ \\
        $\rho$ & Separation parameter for preference probability \\
        $\Delta$ & Separation parameter for reward means \\
        $\tilde{R}_t$ & Preference feedback at time $t$ \\
        $\cD_T$ & History of observations up to time $T$ \\
        $w_a$ & Weight parameter in isotonic regression \\
    \bottomrule
    \end{tabular}
\end{table}

\section{Proof of the regret bound for LOB-UCB}
\label{app:proof_lob_ucb}

\begin{reptheorem}{thm:ucbregret}
    Consider an LOB instance $(s^*, \bmu^*)$ with $k$ arms and $m$ states. Suppose that each possible state $s$ has a partial order $O_s$ and unique optimal arm $a^*_s$. Let $q=|\cA_0| \leq \min(k,m)$. Assume rewards are sub-Gaussian with variance bounded by $\sigma^2$ and define $\Delta_a = \mu^*_{a^*} - \mu^*_a$ with $\Delta_a > 0$ for $a \neq a^*_{s^*}$. Then, Algorithm~\ref{alg:lobucb} using $c=\sigma$ has expected instance-dependent regret
    $$
    R_T \leq \cO(\sum_{a\in \cA_0} \frac{\sigma^2 \log T}{\Delta_a}+q\Delta_{\mathrm{max}}
    )
    $$  
    and worst-case instance-independent regret,
    $$
    R_T \leq \cO(\sigma \sqrt{\min(k,m)T\log T}).
    $$  
\end{reptheorem}
\begin{proof}
To prove the result, we will consider two complementary events, one of which occurs at every round $t$: i) The good event $\cG$, on which the true state $s^* \in \cS_t$, and ii) the bad event $\cG^c$, where $s^*$ may be discarded. On the good event, \lobucb{} performs standard UCB selection over $|\cA_t| \leq q$ arms, and benefits from the same analysis as standard UCB. Bad events each contribute $\Delta_{\mathrm{max}} = \max_{a\neq a^*} \Delta_a$ regret in the worst case, but happen rarely.

\paragraph{On the good event, $\cG$} Recall that $\cS_t$ is defined as the set of states for which feasible parameters intersect with the empirical confidence rectangle, 
$$
\cS_t = \{s : \cH_s \cap \prod_{a=1}^k [L_a(t), U_a(t)]\}~.
$$
Define the good event as the true mean parameter $\bmu^*$ being contained in the empirical confidence intervals $\cG = \forall a : L_a(t) \leq \mu_a^* \leq U_a(t)$. It follows that the true state $s^*$ is contained in $\cS_t$ since $\bmu^* \in \cH_s$ \emph{and} $\bmu^* \in [L(t), U(t)]$. Furthermore, since $\bmu^*$ is consistent with the partial order $O_{s^*}$, it follows that the optimal arm $a^* = \argmax_a \bmu_a^*$ is contained in $\mbox{Max}(s^*)$ and therefore in $\cA_t$. That is, on the good event $\cG$, $a^*$ is considered for play. The regret contributed by $\cG$ is thus determined by (i) how often \emph{other} arms are played on $\cG$, as in standard UCB, and (ii) how often $\cG$ occurs. 

\paragraph{How often does $\cG^c$ occur?} Since regrets are assumed sub-Gaussian, deviations are controlled by standard concentration inequalities. For all $a$, 
$$
\Pr[|\hmu_a - \mu^*_a| > B_a(t)] \leq 2 \exp\left( -\frac{N_a(t)B(t)^2}{2\sigma^2}\right) = 2\exp\left( -\frac{4 c^2 N_a(t)}{2\sigma^2 N_a(t)}\log t \right) \leq 2 t^{-2c^2/\sigma^2}.
$$
Thus, for $c\geq \sigma$, by a union bound over the $q$ arms in $\cA_0$, the expected number of bad events is
$$
\sum_{t=1}^T \sum_{a \in \cA_0} \Pr[|\hmu_a - \mu^*_a| > B_a(t)] \leq 2q \sum_{t=1}^T t^{-2} = q\frac{\pi^2}{3} \in O(q)~,
$$
with $q = |\cA_0|$. On any bad event, the maximum regret is $\Delta_{\mathrm{max}}$.

\paragraph{How much regret to suboptimal arms contribute on $\cG$?}
We bound the number of plays of suboptimal arms $a$ on $\cG_t$. For $a$ to be played, it must be that $U_a(t) \geq U_{a^*}(t)$. By definition, on the good event $\cG_t$, with $c=\sigma$
\begin{align*}
U_{a^*}(t) & = \hmu_{a^*}(t) + c\sqrt{\frac{4\log t}{N_{a^*}(t)}} \\
& \geq \mu^*_{a^*} - \sigma \sqrt{\frac{4\log t}{N_{a^*}(t)}} + c\sqrt{\frac{4\log t}{N_{a^*}(t)}} \\
& = \mu^*_a + \Delta_a \\
& \geq \hmu_a(t) - \sigma\sqrt{\frac{4\log t}{N_{a}(t)}} + \Delta_a \\
& \geq U_a(t) - 2\sigma\sqrt{\frac{4\log t}{N_{a}(t)}} + \Delta_a
\end{align*}
Thus, if $\Delta_a > 2 \sigma\sqrt{\frac{4\log t}{N_{a}(t)}}$, $U_{a^*}(t) > U_a(t)$. This is satisfied whenever $N_{a}(t) > \frac{16\sigma^2}{\Delta_a^2}\log t$. Thus,
\begin{align*}
R_T(\cG) & = \sum_{t=1}^T \mathds{1}[\cG_t]\Delta_{a_t} \\
& = \sum_{t=1}^T \mathds{1}[\cG_t] \sum_{a \in \cA_t} \mathds{1}[a_t=a]\Delta_a \\
& \leq \sum_{a \in \cA_0} N_a(T) \Delta_a \\
& \leq \sum_{a \in \cA_0}\frac{16\sigma^2}{\Delta_a}\log T~.
\end{align*}

\paragraph{Overall regret}
In total, the regret of the algorithm is bounded by
$$
R_T = R_T(\cG) + R_T(\cG^c) \leq \cO(\sum_{a \in \cA_0}\frac{\sigma^2}{\Delta_a}\log T + q \Delta_{\mathrm{max}})~.
$$
Through the standard trick of splitting arms at gap $\epsilon$ and optimizing, we get the gap-independent bound
$$
R_T \leq \cO(\sigma \sqrt{q T \log T}) \leq \cO( \sigma \sqrt{\min(k,m) T \log T})~.
$$

\end{proof}

\section{Experimental details and additional empirical results}
\label{app:experiments}

\begin{figure}[t]
    \centering
    \begin{minipage}[t]{0.48\textwidth}
        \begin{algorithm}[H]
        \caption{Classicals UCB bandit (\ucb{})}
        \label{alg:ucb}
        \begin{algorithmic}[1]
        \State \textbf{Input:} Confidence parameter $c$ 
        \vspace{0.5em}
        \State $\forall a : N_a = 0$, $\hmu_a = 0$
        \For{$t=1, ..., T$}
            \State $I_a \gets \hmu_{a} + \sqrt{\frac{2c^2\log t}{N_{a}}}$ 
            \State Choose action $a_t \gets \argmax_a I_a$
            \State Play arm $a_t$, and observe $r_t$
            \State {\small $\hmu_{a_t} \gets (N_{a_t}\hmu_{a_t} + r_t)/(N_{a_t}+1)$}
            \State $N_{a_t} \gets N_{a_t} + 1$
        \EndFor
        \end{algorithmic}
        \end{algorithm}
    \end{minipage}
    \hfill
    \begin{minipage}[t]{0.50\textwidth}
        \begin{algorithm}[H]
        \centering
        \caption{Gaussian-Gaussian TS bandit  (\ts{})}\label{alg:ts}%
        \begin{algorithmic}[1]
            \State \textbf{Input:} Number of arms $k$, prior means $\mu_a^0$, prior variances $\sigma_a^{2,0}$, confidence parameter $c$
            \vspace{0.5em}
            \For{$a = 1, \dots, k$}
                \State Initialize posterior: $\mu_a \gets \mu_a^0$, $\sigma_a^2 \gets \sigma_a^{2,0}$
            \EndFor
            \For{$t = 1,  ..., T$}
                \For{$a = 1, ..., K$}
                    \State Sample $\tilde{\mu}_a \sim \mathcal{N}(\hmu_a, \sigma_a^2)$
                \EndFor
                \State Choose action $a_t \gets \arg\max_a \tilde{\mu}_a$
                \State Play arm $a_t$, and observe $r_t$
                \State Update posterior for arm $a_t$:
                \State \hspace{1em} $\sigma_{a_t}^2 \gets \left( \frac{1}{\sigma_{a_t}^2} + \frac{1}{\sigma^2} \right)^{-1}$
                \State \hspace{1em} $\hmu_{a_t} \gets \sigma_{a_t}^2 \left( \frac{\hmu_{a_t}}{\sigma_{a_t}^2} + \frac{r_t}{c^2} \right)$
            \EndFor
        \end{algorithmic}
    \end{algorithm}
    \end{minipage}
\end{figure}

\subsection{Details on methods and parameters}
For \mts{} and \mucb{}, we reimplement the algorithms as presented in \citet{hong2020latent}. For \ucb{} and \ts{}, we use implementations of classical Algorithms~\ref{alg:ucb} and \ref{alg:ts}, respectively.

In all experiments, all methods use well-specified confidence parameters $c$, representing the true standard deviation of the rewards, $\sigma$. For the \textbf{synthetic environment}, we use $\sigma=5$, and for \textbf{MovieLens}, we use $\sigma=2$. The \emph{MovieLens} environment has $k=19$ arms, and the synthetic environment varies in this number.

For \lobucb{} and \lobts{}, for all experiments on the \textbf{synthetic environment}, we use $T_{proj}=1$, i.e., we project empirical means every round. For \textbf{MovieLens}, we use $T_{proj}=10$.

\subsection{Details on the MovieLens experiments}
\label{app:movielens}
We create the MovieLens environment based on the 20M version of the rating data sets~\citep{harper2015movielens}, downloaded from Kaggle, \url{https://www.kaggle.com/datasets/grouplens/movielens-20m-dataset}. Recall that instances represent users, actions $a_t$ represent choices of movie \emph{genre} ($k=19$), rewards $r_t$ represent ratings of a movie within the selected genre, and latent states $s$ represent groups of users with similar movie preferences. The custom license used for the MovieLens data can be found at \url{https://grouplens.org/datasets/movielens/}.

For each dataset, we first filter out movies that have less than 200 ratings, and users who have rated less than 200 movies. This leaves 26,826 users 6,236 movies, and 11,905,303 ratings. Each movie is assigned 0 or more genres, with most having multiple assigned. Movies with no genres are assigned the genre ``None'', and movies with more than one genre is assigned a random one from the ones assigned to the specific movie. Based on these, for each user $i$ in the data set, we compute their individual average 5-star rating $\mu^{0,i}_a$ of all movies they have rated within each genre $a$. 

The user-specific genre means are clustered into latent states using k-means clustering with Kendall's Tau as distance and $m$ centroids. The value of $m$ is varied between experiments. Each resulting state $s$ has a total order $O_s$ of the arms corresponding to the centroid's mean parameters. 

To create a well-specified setting for LOB, each user's means $\bmu^{0,i}$ are then projected to comply with the order of their assigned cluster $s_i$, based on Euclidean distance to $\cH_{s_i}$, forming parameters $\bmu^i$. In this setting, we also add a minimum gap $\gamma=0.01$ between the means of any two arms (genres). Rewards for an instance $i$ are generated as $r^i_t \sim \cN(\mu_{a_t}^i, \sigma^2)$ with $\sigma=2$. For the UCB indices and posterior state probabilities of \mucb{} and \mts{}, respectively, we use the state centroids' mean rewards, $\bar{\bmu}_s = \sum_{i : s_i = s} \bmu_i$.

For the bandit evaluation, we take the first $N=100$ random users from the data set and perform $T$ rounds of choosing a movie genre with a random reward (rating) per user. 

\subsection{Projecting means less than every round}
In Figure~\ref{fig:pevery}, we show results on the synthetic environent for how regret scales when the projection steps of \lobts{} are performed only every $T_{proj}$ rounds, for $k=m=10$. We see minimal sensitivity of the regret to this simplification, while it speeds up computation linearly in $T_{proj}$.

\begin{figure}[ht]
    \centering
    \includegraphics[width=0.5\linewidth]{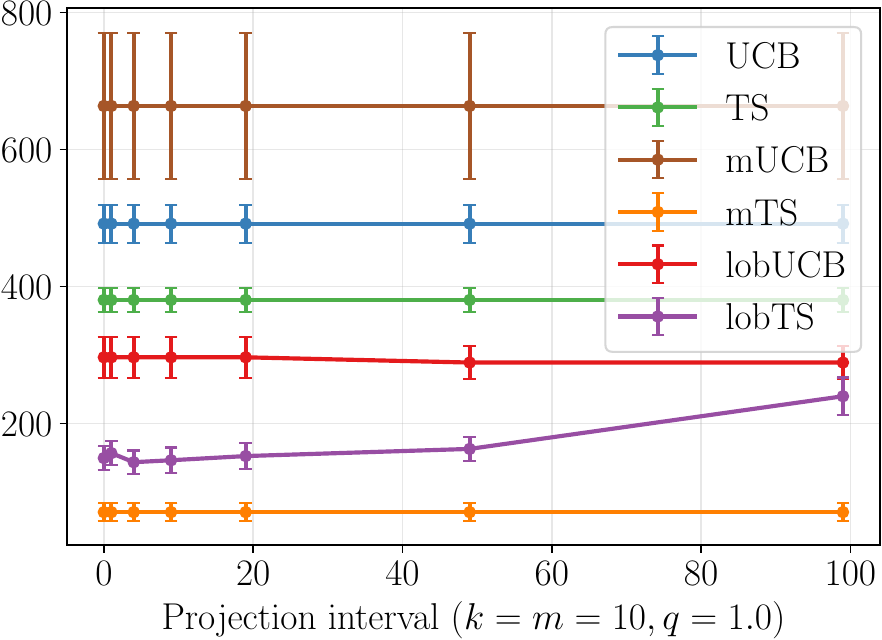}
    \caption{Results on \textbf{Synthetic environment} with $k=m=10$ where \lobts{} only projects every $T_{proj}$ rounds, as a function of $T_{proj}$. Other algorithms are unaffected by this parameter.}
    \label{fig:pevery}
\end{figure}

\subsection{Additional results on the Synthetic environment}

In Figure~\ref{fig:cumulative_regret_synthetic}, we show the cumulative regret on the \textbf{Synthetic} environment for $k=15$ and different settings of $m$.

\begin{figure}[ht]
    \centering
    \begin{subfigure}{0.31\textwidth}
        \centering
        \includegraphics[width=\linewidth]{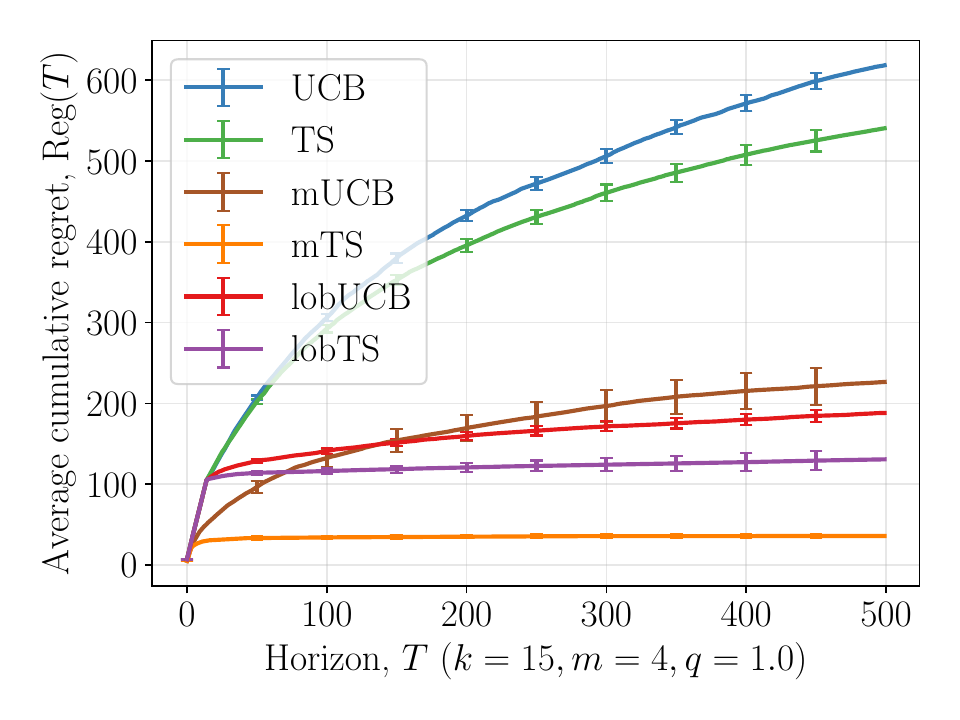}    
        \caption{$m=4$}
    \end{subfigure}
    \begin{subfigure}{0.31\textwidth}
        \centering
        \includegraphics[width=\linewidth]{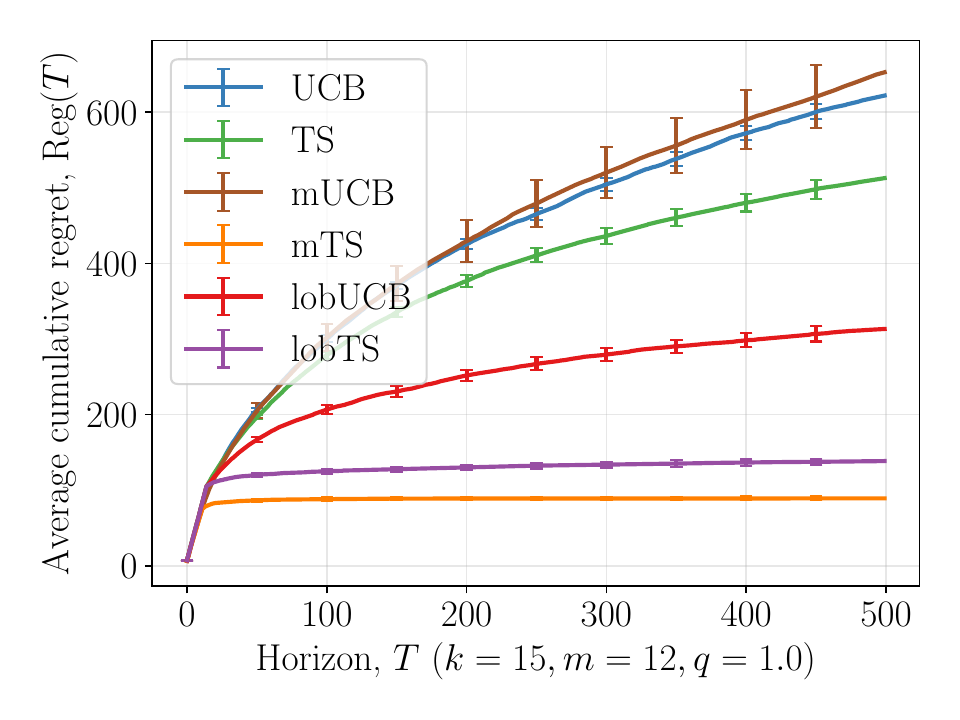}    
        \caption{$m=12$}
    \end{subfigure}
    \begin{subfigure}{0.31\textwidth}
        \centering
        \includegraphics[width=\linewidth]{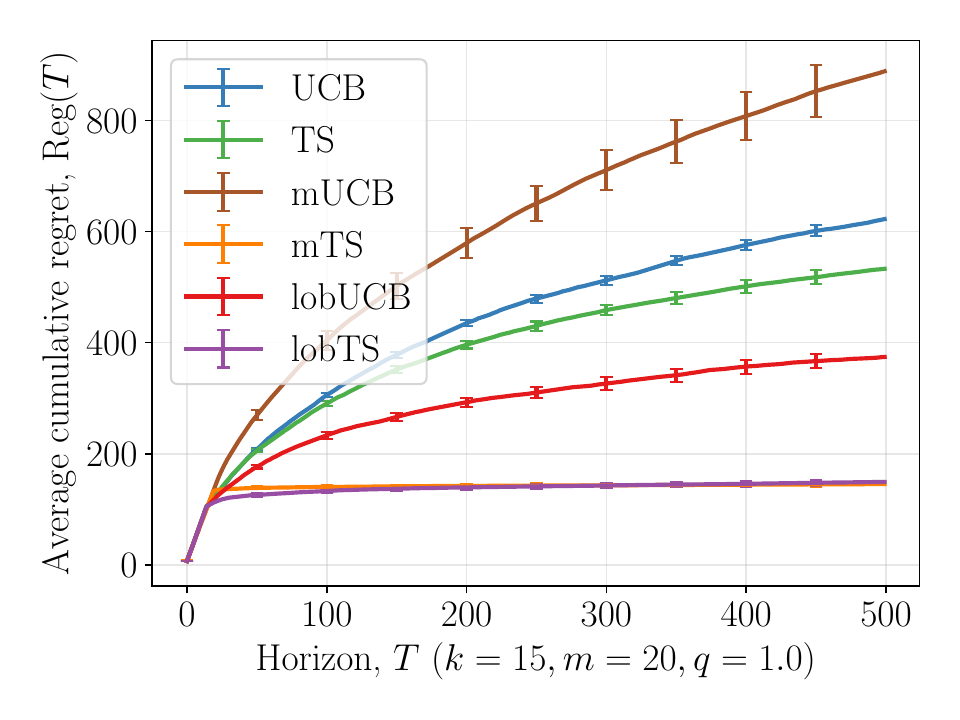}    
        \caption{$m=20$}
    \end{subfigure}    
    \caption{Cumulative regret for the Synthetic environment with $k=15$ and varying $m$.}
    \label{fig:cumulative_regret_synthetic}
\end{figure}

\subsection{Additional results on the MovieLens environment}

In Figure~\ref{fig:instantaneous_regret_movies_m60} we show the instantaneous regret on the \textbf{MovieLens} data set for the well-specified setting, where ground-truth means are projected to follow the order prescribed by their cluster centroid. 
\begin{figure}[ht]
    \centering
    \includegraphics[width=0.5\linewidth]{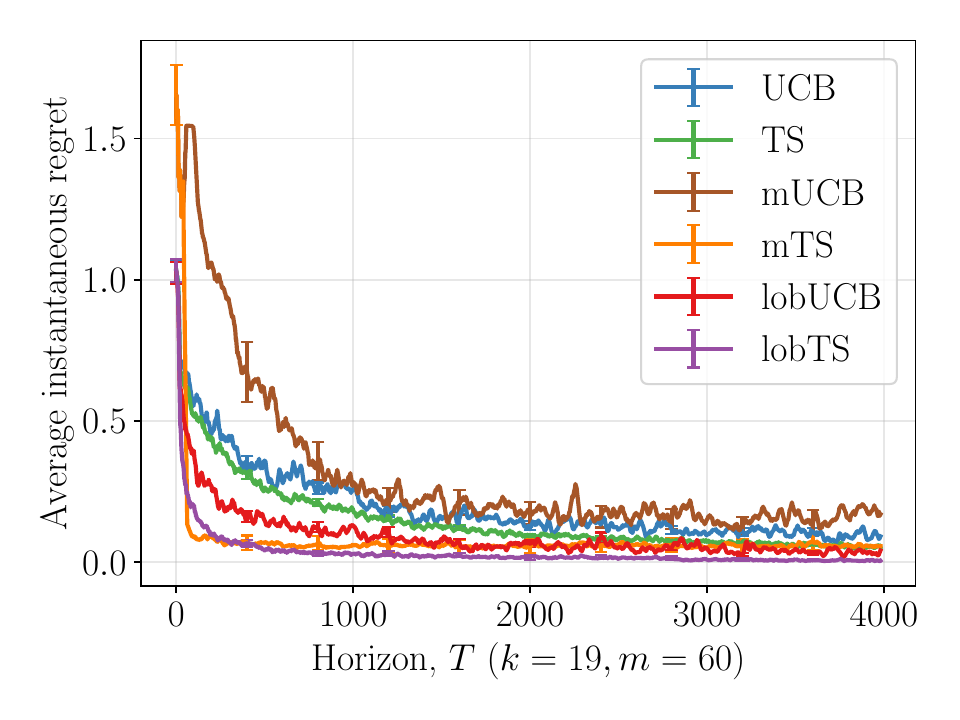}
    \caption{Instantaneous regret for the well-specified setting of the MovieLens environment with $m=60$. Moving average with window size 20 for readability.}
    \label{fig:instantaneous_regret_movies_m60}
\end{figure}

In Figure~\ref{fig:cumulative_regret_synthetic}, we show the cumulative regret on the \textbf{MovieLens} data set for different settings of $m$ in the well-specified Setting A. 

\begin{figure}[ht]
    \centering
    \begin{subfigure}{0.31\textwidth}
        \centering
        \includegraphics[width=\linewidth]{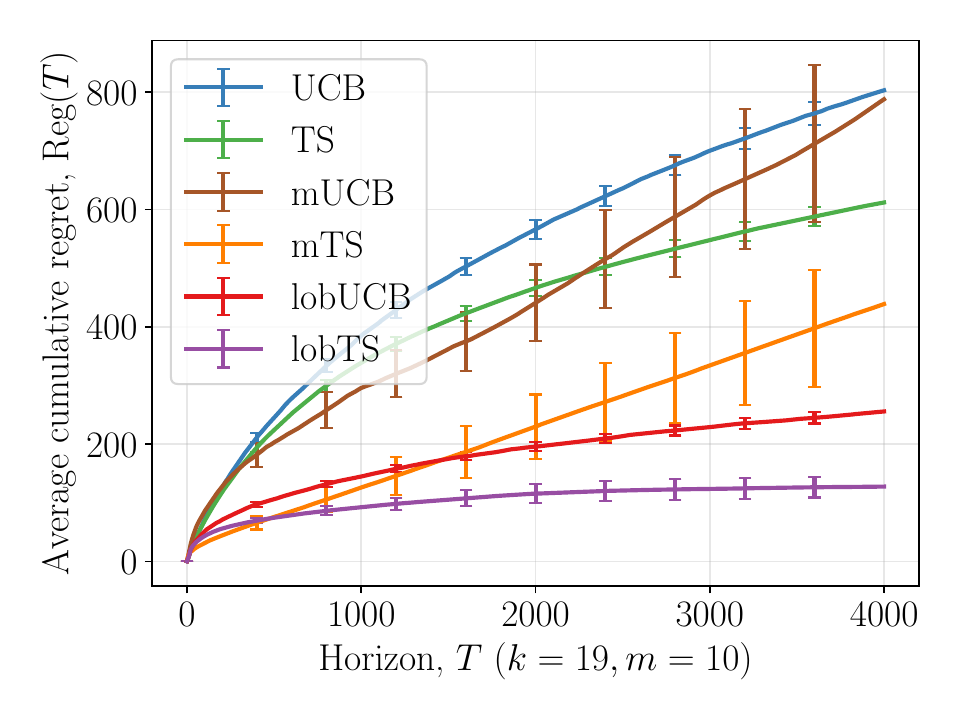}    
        \caption{$m=10$}
    \end{subfigure}
    \begin{subfigure}{0.31\textwidth}
        \centering
        \includegraphics[width=\linewidth]{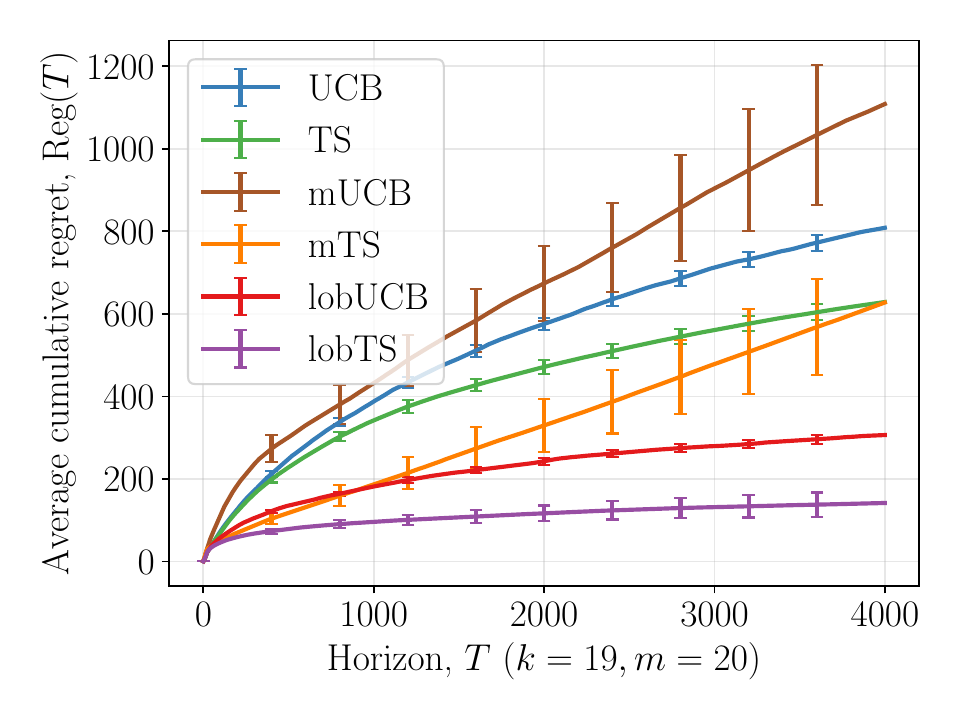}    
        \caption{$m=20$}
    \end{subfigure}
    \begin{subfigure}{0.31\textwidth}
        \centering
        \includegraphics[width=\linewidth]{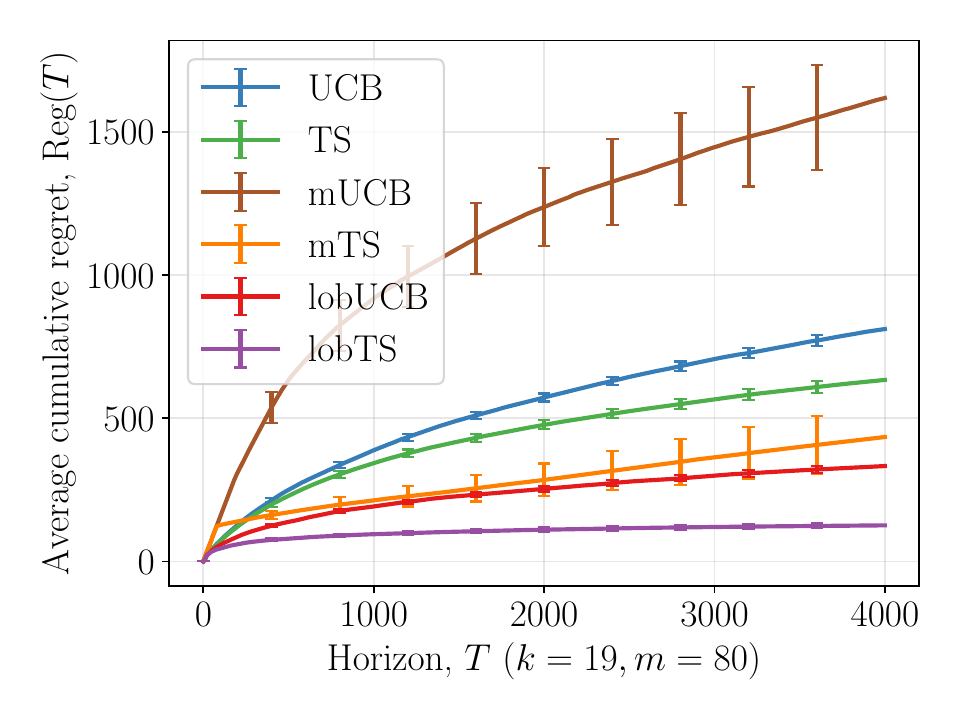}    
        \caption{$m=80$}
    \end{subfigure}    
    \caption{Cumulative regret for the well-specified Setting A of the MovieLens environment with varying $m$.}
    \label{fig:cumulative_regret_movies}
\end{figure}

In Figure~\ref{fig:instantaneous_regret_movies_real}, we show the instantaneous regret on the \textbf{MovieLens} data set for different settings of $m$. When $m$ grows, the degree of misspecification for the LOB algorithms decreases, and the regret bias does as well. For long horizons, $M>2000$, even at $m=160$, there is still remaining bias, and unstructured algorithms overtake LOB algorithms in the long run. However, for shorter horizons, feasible to implement in real-world personalization tasks, the structured algorithms are preferable even in the misspecified case. 

\begin{figure}[ht]
    \centering
    \begin{subfigure}{0.31\textwidth}
        \centering
        \includegraphics[width=\linewidth]{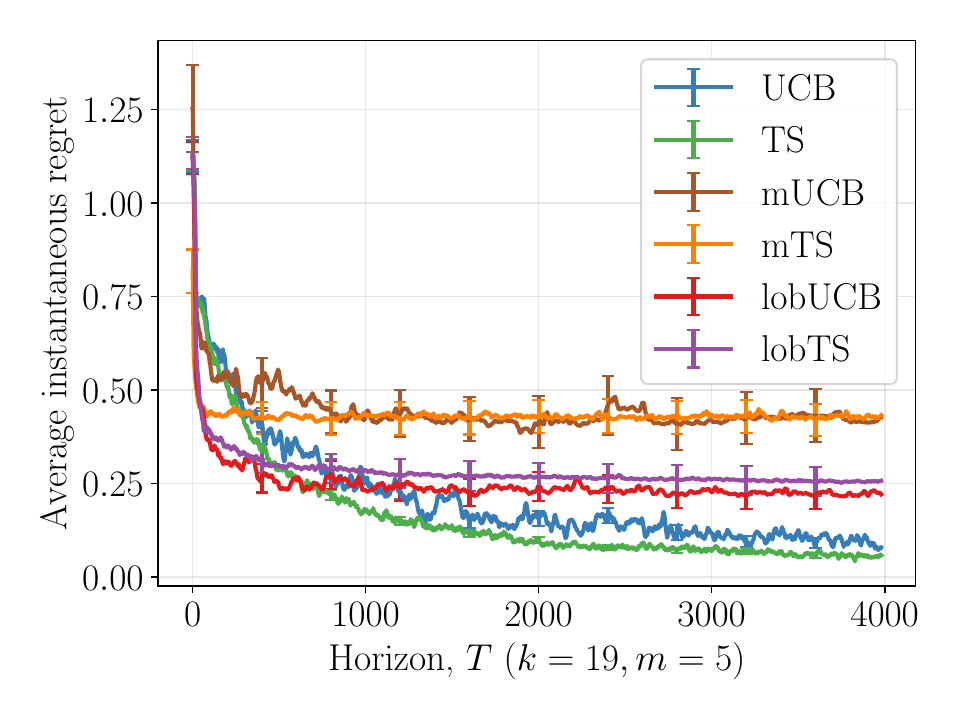}    
        \caption{$m=5$}
    \end{subfigure}
    \begin{subfigure}{0.31\textwidth}
        \centering
        \includegraphics[width=\linewidth]{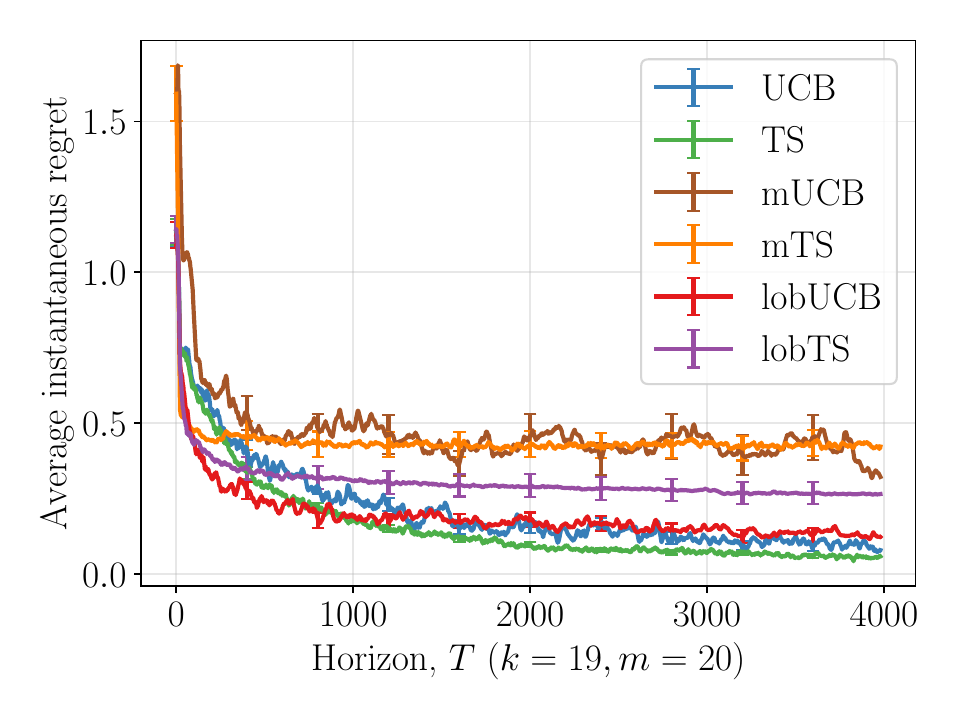}    
        \caption{$m=20$}
    \end{subfigure}
    \begin{subfigure}{0.31\textwidth}
        \centering
        \includegraphics[width=\linewidth]{fig/instantaneous_regret_movies_real_m140_upd.pdf}    
        \caption{$m=140$}
    \end{subfigure}    
    \caption{Instantaneous regret for the misspecified (real-world means) setting of the MovieLens environment with varying $m$.}
    \label{fig:instantaneous_regret_movies_real}
\end{figure}

\subsection{Computational infrastructure and code}
The synthetic simulations were run on a 2023 Apple M2 Pro laptop with 16 GB RAM. Code to reproduce the experiments will be made available upon publication of the paper.

\section{Argument that restricting the arms of UCB achieves a worst-case regret comparable to mTS and mUCB}
\label{app:topm}

\begin{thmprop}\label{prop:regret_upper}Consider the following algorithm for LB or LOB problems where the optimal arm is known for each state. 
Whenever $k>m$, restrict the action set to the subset $\cA_\cS^*$ of $u \leq m$ arms that are optimal in at least one latent state, $\cA_\cS^* = \{a\in [k] : \exists s\in [m]: a_s^* = a\}$, and run a standard MAB algorithm restricted to $\cA_\cS^*$. When $k\leq m$, run a standard MAB algorithm on $\cA = [k]$. This procedure achieves $O(\sqrt{\min(k, m)T})$ regret in the worst case on the latent bandit and latent order bandit problems whenever the latent order of each state implies the best arm.
\end{thmprop}
\begin{proof}
    The procedure uses $\min(|\cA|, |\cA_\cS^*|) \leq \min(k, m)$ arms, and the optimal arm for any latent bandit or LOB instance is contained in either action set. We can directly apply standard regret bounds for MAB algorithms (see e.g.,~\citet{lattimore2020bandit}) to the restricted action sets.
\end{proof}

Notably, the procedure in Proposition~\ref{prop:regret_upper} does not use knowledge of the set of possible reward \emph{means}, only the set of possible best arms. On its face, it may seem that the \mts{} algorithm of \citet{hong2020latent} is matched by this simple algorithm, and improved upon when $m>k$. However, as we see empirically in Section~\ref{sec:experiments}, this is far from true: \mts{} achieves substantially better performance by exploiting reward structure, even though this is not yet explained by theory. However, this is not achievable in the LOB setting since the mean parameters must be estimated during exploration.

\section{Proof that isotonic regression solves the constrained MLE problem}

We prove that a weighted isotonic regression problem can solve the parameter projections in Algorithms~\ref{alg:lobucb}--\ref{alg:lobts} for total orders.
\label{app:isotonic}
\begin{thmprop}\label{prop:isotonic}
    Let $n_a = \sum_{t=1}^T\mathds{1}[a_t = a]$ and define $\bar{w}_a = \frac{n_a}{\sigma^2}$. Next, let $\cI = (i_1, ..., i_k)$ be the rank of each arm corresponding to the total arm order $O_s$ of latent state $s$. Then, the solution to the isotonic regression problem with outcomes $y_a = \frac{1}{n_a}\sum_{t : a_t=a}r_t$ and sample weights $\bar{w}_a$ 
    \begin{equation*}
    \begin{aligned}
    & \underset{\bmu \in \bbR^d}{\text{minimize}}
    & & \sum_{a=1}^k \bar{w}_a (\mu_a - y_a)^2 
    & \text{subject to}
    & & \mu_{i_k} \leq \mu_{i_{k-1}} \leq ... \leq \mu_{i_1}
    \end{aligned}
    \end{equation*}
    solves the constrained MLE problem in Algorithm~\ref{alg:lobucb}, ln.5.
\end{thmprop}
\begin{proof}
    We show that the isotonic regression objective is equal to the problem on ln.5. of Algorithm~\ref{alg:lobucb} up to a constant. We have 
    \begin{align*}
        \sum_{a=1}^k \bar{w}_a (\mu_a - y_a)^2 & = \sum_{a=1}^k \bar{w}_a (\mu_a^2 - 2y_a\mu_a + y_a^2) 
         = \sum_{a=1}^k \bar{w}_a (\frac{1}{n_a}\sum_{t:a_t=a}[\mu_a^2 - 2r_t\mu_a] + y_a^2) \\
        & = \sum_{a=1}^k \frac{1}{\sigma_a^2} (\sum_{t:a_t=a}[\mu_a^2 - 2r_t\mu_a + r_t^2] + y_a^2 - \sum_{t:a_t=a}\frac{r_t^2}{n_a}) 
         = \sum_{t=1}^T \frac{(\mu_{a_t}^2 - r_t^2)}{\sigma_{a_t}^2} + C~,
    \end{align*}
    where $C$ is a constant w.r.t. $\bmu$. Thus minimizing the LHS and RHS yields the same solution. 
\end{proof}
For partial orders, only the constraint set changes. 

\section{A note on relative feedback}
\label{app:relative}

Dueling bandits~\citep{yue2009interactively, sui2018advancements, bengs2021preference} are the simplest-to-analyze as rewards are observed in the same format as the latent state--as relative preference feedback. Let $\mathcal{D}_T = ((a_t, a'_t, \tilde{r}_t))_{t=1}^T$ be a sequence of preference feedback events where $a_t, a'_t \in [m]$ are two competing actions and $\tilde{r}_t \in \{0,1\}$ indicates which action was preferred. Assuming that there are latent, noisy and continuous rewards $r_t, r'_t$ for the two actions, let $\tilde{R}_t = \mathds{1}[R_t \geq R'_t]$ indicate noisy preferences for $a_t$ over $a'_t$ and $\tilde{r}_t$ its realization. 

If for any latent state $s \in [m]$, there exists an (unknown) margin parameter $\rho > 0$ such that most of the time, with a margin $\rho$, the reward for $a$ is higher than the reward for $a'$, if $a$ is preferred to $a'$ 
then, 
$$
p(\tilde{R}_t = 1 \mid a_t, a'_t, s) \leq \left\{
\begin{array}{ll}
\frac 1 2 -\rho, & a_t \preceq_s a_t' \\
1, & \mbox{otherwise}
\end{array}\right.
$$
provides a crude upper bound on the likelihood of a single reward from $\mathcal{D}_T$ under $s$. In other words, the probability of observing $r_t > r'_t$ is less than $1/2-\rho$ if the action $a'_t$ has lower rank than $a_t$. Defining $n_i(s)$ to be the number of observed inversions of the rank imposed by $s$
$
n_i(s) = \sum_{t=1}^T (\mathds{1}[a_t \succ_s a_t'] \neq \mathds{1}[r_t \leq r'_t])~
$ 
then, we can upper bound the full likelihood under $s$ as 
\begin{equation}\label{eq:rel_lik_bound}
p(\mathcal{D}\mid s) \leq (\frac 1 2 -\rho)^{n_i(s)}\cdot 1^{T-n_i(s)} \leq 2^{-n_i(s)}
\end{equation}
and, likewise, the posterior $p(s \mid \mathcal{D}) \propto p(\mathcal{D}\mid s)p(s)$.
Bounding the posterior as in \eqref{eq:rel_lik_bound} allows us to rule out candidate latent states as the probability decays with the number of inversions, and a similar posterior sampling algorithm like \lobTS{} for preference feedback can be obtained.

\subsection{Absolute to relative feedback}
\label{app:absolute_feedback}
Absolute reward feedback can always be turned into preference feedback. For example, plays and rewards $(a_1, r_1), (a_2, r_2), (a_3, r_3), (a_4, r_4)$ can be paired up consecutively: $(a_1, a_2, \mathds{1}[r_1>r_2]), (a_3, a_4, \mathds{1}[r3>r4])$. This ensures that different pairs comprise independent events, unlike, say, an all-pairs comparison. However, this procedure is likely very inefficient, statistically. Rewards obtained for each action provide much more information than is used by considering the pairs of most recent actions. Moreover, the algorithm does not rule out playing the same action twice, which means it is prone to getting stuck in local optima. 


\section{Rarity of similar total orders in LOB}
\label{app:similarity_orders}
\begin{observation} \textbf{Rarity of Similar Total Orders}.
Consider $k$ actions in the LOB framework, with $O_{s_1}$ and $O_{s_2}$ as two distinct preference orders drawn randomly from the set of all permutations of $k$ actions. The probability that $O_{s_1}$ and $O_{s_2}$ differ in exactly two positions is:
$$
P(|\{i \mid O_{s_1}(i) \neq O_{s_2}(i)\}| = 2, i\in[k]) = \frac{\binom{k}{2}}{k! - 1}
$$
\end{observation}
\begin{proof}
Consider two distinct random total orders  $O_{s_1}$ and $O_{s_2}$ of $k$ actions, here represented as permutations, rather sets of than pairwise relations. We need to find the probability that they differ in exactly two positions, i.e., $|\{i \mid O_{s_1}(i) \neq O_{s_2}(i)\}| = 2, i\in[k]$. Define the relative permutation $\tau = O_{s_2}^{-1} \circ O_{s_1}$. The differing positions are the non-fixed points of $\tau$ (where $\tau(i) \neq i$), so $\tau$ must be a \textit{transposition}, swapping two elements. The number of possible transpositions is $\binom{k}{2}$. The total number of ordered pairs $(O_{s_1}, O_{s_2})$ with $O_{s_1} \neq O_{s_2}$ is $k! \cdot (k! - 1)$. For each transposition $\tau$, there are $k!$ pairs where $O_{s_2} = O_{s_1} \circ \tau$, since $O_{s_1}$ can be any permutation. Thus, the number of favorable pairs is $\binom{k}{2} \cdot k!$.

Therefore, the probability is:
$$
\frac{\binom{k}{2} \cdot k!}{k! \cdot (k! - 1)} = \frac{\binom{k}{2}}{k! - 1}
$$

\end{proof}

\section{Identifiability of the latent variable model}
Under mild assumptions, the set of possible preference orders are \emph{identifiable} from historical sequences of actions and rewards. All we need is access to $\mu_{i,a} \coloneqq \E[R_{i,a}]$ for any instance $i$ and action $a$, and we can extract the preference orders for instance $i$. By assumption, these will make up a set of $m$ unique orders, which is the set we are looking for. 

With assumptions on the probability of observing each action for each instance, the set of orders is also \emph{estimable}. Consider a data set of data from $N$ instances $i=1, ..., N$, each represented by a random observation sequence $(a_{i,1}, r_{i,1}), ..., (a_{i,T}, r_{i,T})$ and an unobserved latent state $s_i$. Each subject has reward means $\bmu_i = (\mu_{i,1}, ..., \mu_{i,k})$ consistent with the latent state, $\bmu_i \in \cH_{s_i}$. Further, assume that for all $i,a$, $|\mu_{i,a} - \mu_{i,a'}| \geq \Delta$ for some $\Delta > 0$. Finally, assume for expositional purposes that $R_{i,t} \sim \cN(\mu_{i,A_t}, \sigma^2)$ for a fixed $\sigma>0$. Let $n_{i,a} = |\{t \leq T : a_{i,t} = a\}|$ and define $n_{\min} = \min_{i,a} n_{i,a}$. The same result holds for general-Gaussian random variables.  We can use the standard sub-Gaussian tail bound to guarantee that 
$$
\mathrm{Pr}(|\hat{\mu}_{i,a} - \mu_{i,a}| > \Delta) \leq 2\exp\left(-\frac{n_{\min}\Delta^2}{2\sigma^2}\right)~.
$$
By the union bound, 
$$
\mathrm{Pr}(\max_{a \in [k]} |\hat{\mu}_{i,a} - \mu_{i,a}| > \Delta) \leq 2k\exp\left(-\frac{n_{\min}\Delta^2}{2\sigma^2}\right)~.
$$
This is an upper bound on the probability that the order of a single instance is misestimated \emph{at all}. For all of the orders to be exact, we need to multiply the RHS by $N$. Thus, with probability at least $1-\delta$, if 
$$
n_{\min} \geq 2\frac{\sigma^2}{\Delta^2}\log\left(\frac{2kN}{\delta} \right)~, 
$$
all instance orders are correct. If all instance orders are correct, there exists exactly $m$ unique orders with high probability, provided that $\min_s p(S=s)$ and $N$ are sufficiently large. These are the latent orders we are looking for. If $n_{\min}$ is a random variable, such as if the historical data was also collected by a bandit algorithm we need to take that into account as well.

\section{Broader impact }
Although the methods in this work are primarily methodological and do not have direct practical implications, they are designed with real-world sequential decision-making in mind. However, any application of our method in general personalization must be made with caution and sufficient guard rails appropriate for the specific problem. 


\end{document}